\documentclass[twoside,11pt]{article}

\usepackage{jmlr2e}

\usepackage{amsmath}

\usepackage{multirow}
\usepackage{algorithmic}
\usepackage{algorithm}
\usepackage{bbm}

\newtheorem{problem}{Problem}

\ShortHeadings{Simplification of Forest Classifiers and Regressors}{Nakamura and Sakurada}
\firstpageno{1}

\begin{document}

\title{Simplification of Forest Classifiers and Regressors by Sharing Branching Conditions\thanks{The preliminary version of this paper appeared in ECMLPKDD2019 \citep{NS2019}.}}

\author{\name Atsuyoshi Nakamura\thanks{Corresponding Author.} \email atsu@ist.hokudai.ac.jp 
%  \addr Graduate School of Information Science and Technology\\
%        Hokkaido University\\
%        Kita 14, Nishi 9, Kita-ku, Sapporo, Hokkaido, 060-0814, Japan
       \AND
       \name Kento Sakurada \\%\email k.sakurada.e.contact@gmail.com \\
       \addr Graduate School of Information Science and Technology\\
        Hokkaido University\\
        Kita 14, Nishi 9, Kita-ku, Sapporo, Hokkaido, 060-0814, Japan}

%\editor{Kevin Murphy and Bernhard Sch{\"o}lkopf}
\editor{}

\maketitle

\begin{abstract}%   <- trailing '%' for backward compatibility of .sty file
  We study the problem of sharing as many branching conditions of a given forest classifier or regressor as possible while keeping classification performance. As a constraint for preventing from accuracy degradation, we first consider the one that
  the decision paths of all the given feature vectors must not change. For a branching condition that a value of a certain feature is at most a given threshold, the set of values satisfying such constraint can be represented as an interval. Thus, the problem is reduced to the problem of finding the minimum set intersecting all the constraint-satisfying intervals for each set of branching conditions on the same feature. We propose an algorithm for the original problem using an algorithm solving this problem efficiently. The constraint is relaxed later to promote further sharing of branching conditions by allowing decision path change of a certain ratio of the given feature vectors or allowing a certain number of non-intersected constraint-satisfying intervals. We also extended our algorithm for both the relaxations. The effectiveness of our method is demonstrated through comprehensive experiments using 21 datasets (13 classification and 8 regression datasets in UCI machine learning repository) and 4 classifiers/regressors (random forest, extremely randomized trees, AdaBoost and gradient boosting).
\end{abstract}


\section{Introduction}

Edge computing is a key for realization of real time and low loading intelligent data processing.  
Highly accurate compact classifiers or regressors are necessary for edge computing devices.
To produce compact models, researches by various approaches have been done \citep{MMHKAH2021} including
directly learning compact models \citep{Gupta2017,KGV2017} and post processing like quantization and compression.

Ensemble classifiers and regressors are very popular predictors due to their good accuracy and robustness.
The most frequently used ensembles are tree ensembles that include 
random forests \citep{Breiman2001}, extremely randomized trees \citep{GEW2006}, AdaBoost \citep{FREUND1997119}
gradient boosted regression trees \citep{Friedman00}, and XGBoost \citep{CG2016}.

In this paper, we study simplification of tree ensembles for post-processing compression.
As such simplification method, pruning \citep{RF-pruning} has been studied for more than three decades.
Most existing pruning techniques reduce the number of branching nodes or trees.
We, however, consider post-compression methods that simplify a given tree ensemble by sharing branching conditions without changing the structure of the tree ensemble.
Branching condition sharing is corresponding to comparator sharing which contributes to compact hardware implementation of tree ensemble predictors \citep{JSN2018,ISNMT2020}.
Though comparator sharing by threshold clustering has been already studied by \citet{JSN2018} before our research,
we pursue a more sophisticated method than simply integrating similar threshold values.

Firstly, we formalize our simplification problem as the problem of minimizing the number of distinct branching conditions in a tree ensemble by condition sharing under the constraint that any decision path of the given feature vectors must not be changed.
We consider the case that all the features are numerical and every branching condition is
the condition that the value of a certain feature is at most a certain threshold $\theta$. 
Then, the range of $\theta$ that satisfies the above constraint becomes some interval $[\ell,u)$,
and the problem can be reduced to each feature's problem of obtaining a minimum set that intersects all the range intervals of its thresholds. We propose Algorithm Min\_IntSet for this reduced problem and prove its correctness.   
We also develop Algorithm Min\_DBN for our original problem using Min\_IntSet to solve the reduced problem for each feature.

To promote further sharing of branching conditions, in addition to simplely using of bootstrap training samples per tree, we propose two extended problems by relaxation of the restriction.
One is relaxation by allowing the existence of feature vectors whose decision paths change.
We introduce a parameter $\sigma$ called ``path changeable rate'' and allow $100\sigma$\% feature vectors passing through each node
to change. The other is relaxation by allowing the existence of non-intersecting intervals in the reduced problem.
We introduce a parameter $c$, the number of allowable exceptions, and allow at most $c$ intervals that are not intersected by the solution set in the reduced problem. For both extended problems, we provide solution algorithms and their time and space complexities.

The effectiveness of our algorithm Min\_DBN and its extensions were demonstrated using 13 classification and 8 regression datasets in the UCI machine learning repository \citep{Dua:2017}, and four major tree ensembles, random forest (RF), extremely randomized trees (ERT), AdaBoost (ABoost) and gradient boosting (GBoost), whose implementations are provided by scikit-learn python library \citep{scikit-learn}.
For each dataset, we conducted a 5-fold cross validation,
and four folds were used to generate the tree ensemble classifiers and regressors, to which our simplification algorithms
were applied. The remaining one fold was used to check accuracy degradation.
Medians of the reduced size ratios among all the 21 datasets are 18.1\%, 2.1\%, 45.9\%, and 85.9\% for RF, ERT, ABoost, and GBoost, respectively,
while all the accuracy degradations are within 1\% except one dataset for ERT.
Among our two extensions, the method allowing exceptional intervals performs better than the method allowing exceptional feature vectors
in terms of size ratio and accuracy degradation, though the computational cost of the former method is heavier than that of the latter method.
As for the comparison with the k-means clustering-based method, the method allowing exceptional intervals found more accurate pareto optimal predictors in our experiment using systematically selected parameters.

\subsection{Related Work}

Pruning of tree classifiers and regressors is a very popular technique as a countermeasure of overfitting to training data
\citep{Esposito1997}. Pruning is a kind of simplification of tree classifiers and regressors which reduces the number of nodes
by replacing subtrees with leaves. Various pruning methods have been developed around 30 years ago \citep{Quinlan87,NB87,BreiFrieStonOlsh84,CB91,Minger89,BB94}. 
In the pruning of tree ensembles,
in addition to directly applying pruning methods for a tree to each component tree \citep{WA01},
node reduction considering prediction and accuracy of the whole ensemble classifier was developed;
node importance-based pruning \citep{JWG17} and the pruning using 0-1 integer programming treating both error and computational cost \citep{NWJS16}.
The reduction of component trees was also realized by maximizing the diversity of component trees \citep{ZBS06} or minimizing the sum of voting integer weights for component trees without changing predictions of given feature vectors \citep{MaekawaNK20}.
The minimum single trees that are equivalent to given tree ensembles were reported to become smaller than the original tree ensembles
without accuracy degradation by applying pruning that is based on given feature vectors to the equivalent single trees \citep{VidalS20}.

Any simplification method for tree ensembles simplifies the input space partition of a given tree ensemble.
As a direct simplification of the partition, reduction of the partitioned regions was done by the Bayesian model selection approach \citep{HH18}. Branching condition sharing is one of the indirect simplifications of input space partition.
Threshold clustering \citep{JSN2018} is a simple way to do the task.
We propose a more sophisticated way that is suitable to find decision boundaries without accuracy degradation.  

\section{Problem Setting}

Let $X$ be $d$-dimensional real feature space $\mathbb{R}^d$ and let $Y$ be a categorical space $C=\{1,2,\dots, \ell\}$ or a real-valued space $\mathbb{R}$.
A \emph{decision tree} is a rooted tree in which a branching condition is attached to each internal node
and a value in $Y$ is assigned to each leaf node.
A decision tree represents a function from $X$ to $Y$ whose value $y\in Y$ for $\mathbf{x}\in X$ is the value assigned to the leaf node
that can be reached by staring from the root node and choosing the left or right child node at each internal node depending on the branching condition attached to the node. 
Here, we assume that each internal node has just two child nodes and the attached branching condition is in the form of $x_i\leq \theta$, where $x_i$ is the $i$th dimensional value of vector $\mathbf{x}\in X=\mathbb{R}^d$.
For a given feature vector $\mathbf{x}$, the left child node is chosen at each internal node if its branching condition $x_i\leq \theta$ is satisfied
in the function-value assignment process using a decision tree. Otherwise, the right child node is chosen.
The \emph{$\mathbf{x}$'s path in a decision tree $T$} is the path from the root node to the reached leaf node in the function-value assignment process. 
We let a pair $(i,\theta)$ of a feature id $i$ and a threshold $\theta$ denote the branching condition $x_i\leq \theta$.
A majority voting classifier of decision trees for a categorical space $Y$ is called a \emph{decision forest classifier},
and an average regressor of decision trees for a real-valued space $Y$ is called a \emph{decision forest regressor}.

We consider the following problem of minimizing the number of distinct branching conditions (NDC).

\begin{problem}[NDC minimization problem]\label{prob:DBC}
  For a given tree ensemble $\{ T_{1},...T_{m} \}$, a given set of feature vectors $\{\mathbf{x}_{1},...,\mathbf{x}_{n} \}$, minimize the number of distinct branching conditions $(i,\theta)$ by changing the values of some $\theta$ without changing feature vectors' paths passing through each internal node of each decision tree $T_j$ $(j=1,\dots,m)$.  
\end{problem}

\begin{figure}[tb]
\centering
\includegraphics[width=1.0\linewidth]{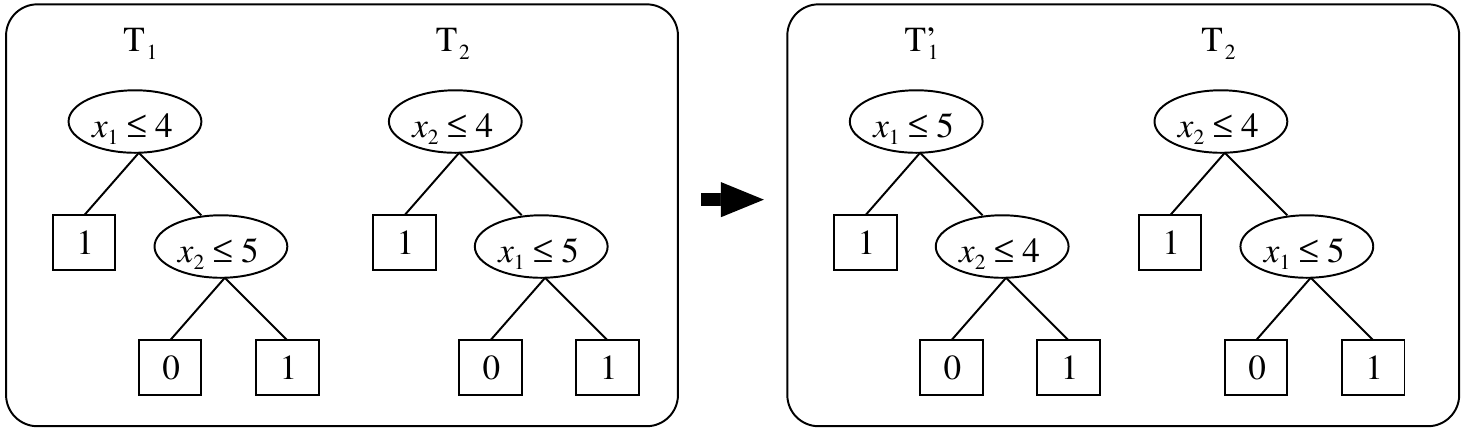}
\caption{The number of distinct branching conditions $(1,4), (1,5), (2,4), (2,5)$ in decision forest $\{T_1,T_2\}$
  can be reduced to $(1,5), (2,4)$ by changing conditions $(1,4)$ and $(2,5)$ in $T_1$ to $(1,5)$ and $(2,4)$, respectively,
  without changing the path of any feature vector in $\{(1,1),(2,7),(7,2),(8,8)\}$.}
\label{fig:example}
\end{figure}

\begin{example}
Consider Problem~\ref{prob:DBC} for a decision forest $\{T_1,T_2\}$ in Figure~\ref{fig:example} and a feature vector set $\{(1,1),(2,7),(7,2),(8,8)\}$. The Distinct branching conditions $(i,\theta)$ in decision forest $\{T_1,T_2\}$ are the following four:
\begin{equation*}
(i,\theta) =(1,4),(1,5),(2,4),(2,5).
\end{equation*}
The branching conditions $(1,4)$ and $(2,5)$ in $T_1$ can be changed to $(1,5)$ and $(2,4)$, respectively, 
without changing the path of any feature vector in the given set $\{(1,1),(2,7),(7,2),(8,8)\}$.
The decision tree $T'_1$ in Figure~\ref{fig:example} is the one that is made from $T_1$ by this branching-condition change.
Decision forest $\{T'_1,T_2\}$ has two distinct conditions $(1,5)$ and $(2,4)$, and the forest is a solution to Problem~\ref{prob:DBC}. 
\end{example}

\begin{remark}
  Decision forest $\{T_1,T_2\}$ in Figure~\ref{fig:example} can be outputted by a decision forest learner with training samples $((x_1,x_2),y)=((1,1),1),((2,7),0),((7,2),0),((8,8),1)$.
  Assume that two sets of bootstrap samples are $D_1=\{((1,1),1),((7,2),0),((8,8),1)\}$ and $D_2=\{((1,1),1),$ $((2,7),0),((8,8),1)\}$,
  and all the features are sampled for both the sets.
  In the implementation that the middle points of adjacent feature values are used as threshold candidates for branching conditions,
  CART algorithm can output $T_1$ for $D_1$ and $T_2$ for $D_2$.
  Decision tree $T'_1$ in Figure~\ref{fig:example} has the same Gini Impurity as $T_1$  at each corresponding branching node for the set of samples $D_1$.
\end{remark}

\section{Problem of Minimum Set Intersecting All the Given Intervals}

At each branching node with condition $(i,\theta)$,
the range in which $\theta$ can take a value without changing given feature vectors' paths passing through the node,
becomes interval $[\ell,u)$.
Thus, the problem of minimizing the number of distinct branching conditions in a decision forest can be solved by 
finding clusters of conditions $(i,\theta)$ whose changeable intervals have a common value for each feature $i$.
Thus, solving Problem~\ref{prob:DBC} can be reduced to solving the following problem for each feature $i$. 

\begin{problem}[Minimum intersecting set problem]\label{prob:minintset}
For a given set of intervals $\{[\ell_{1},u_{1}),$ $\dots,[\ell_{p},u_{p}) \}$, find a minimum set that intersects all the intervals $[\ell_j,u_j)$ ($j=1,\dots,p$).
\end{problem}

\begin{algorithm}[tbh]
\caption{Min\_IntSet}\label{Min_IntSet}
\begin{algorithmic}[1]
\REQUIRE $\{ [\ell_{i},u_{i}) | i \in I  \}$ : Non-empty set of intervals
\ENSURE \begin{minipage}[t]{12cm}
$ \{ s_{1},...,s_{k} \}$ : Minimum set satisfying $\{ s_{1},\dots,s_{k} \}\cap[\ell_{i},u_{i}) \neq \emptyset$ $(i\in I)$\\
$\{ I_{1},..I_{k} \}$ : $I_{j}=\{i \in I | s_{j} \in [\ell_i,u_{i})\} \  (j=1,...,k)$
\end{minipage}
\STATE $[\ell_{i_1},u_{i_1}),...,[\ell_{i_p},u_{i_p})\leftarrow$ list sorted by the values of $\ell_i$ in ascending order.
\STATE $k \leftarrow 1, t_1 \leftarrow u_{i_1}, b_1\leftarrow 1$
\FOR {$j=2$ to $p$}
\IF {$\ell_{i_j} \geq t_k$}
\STATE $s_k \leftarrow { \frac{ \ell_{i_{j-1}}+t_k }{2} }$
\STATE $I_{k} \leftarrow \{ i_{b_k},..., i_{j-1} \}$
\STATE $k \leftarrow k+1, t_k \leftarrow u_{i_j}, b_k \leftarrow j$
\ELSIF{ $u_{i_j} < t_k$}
\STATE $t_k \leftarrow u_{i_j}$
\ENDIF
\ENDFOR
\STATE $s_{k} \leftarrow { \frac{\ell_{i_p}+t_k }{2} }, I_{k}  \leftarrow \{i_{b_k},...,i_{p} \} $
\RETURN $\{ s_{1},...,s_{k} \}, \{ I_{1},..,I_{k} \} $
\end{algorithmic}
\end{algorithm}

We propose Min\_IntSet (Algorithm \ref{Min_IntSet}) as an algorithm for Problem~\ref{prob:minintset}.
The algorithm is very simple.
First, it sorts the given set of intervals $\{[\ell_1,u_1),\dots,[\ell_p,u_p)\}$ by lower bound $\ell_i$ in ascending order (Line 1).
    For the obtained sorted list $([\ell_{i_1},u_{i_1}),\dots,[\ell_{i_p},u_{i_p}))$,
starting from $k=1$ and $b_1=1$, the algorithm finds the $k$th point $s_k$ by calculating 
the maximal prefix $([\ell_{i_{b_k}},u_{i_{b_k}}),\dots,[\ell_{i_{j-1}},u_{i_{j-1}}))$ of the list $([\ell_{i_{b_k}},u_{i_{b_k}}),\dots,[\ell_{i_p},u_{i_p}))$ that contains non-empty intersection
        \[
\bigcap_{h=b_k}^{j-1}[\ell_{i_h},u_{i_h})=[\ell_{i_{j-1}},\min_{h=b_k,\dots,j-1} u_{i_h}),
        \]
        and $t_k$ is updated such that $t_k=\min_{h=b_k,\dots,j-1} u_{i_h}$ holds (Line 9).
        The algorithm can know the maximality of the prefix $([\ell_{i_{b_k}},u_{i_{b_k}}),\dots,[\ell_{i_{j-1}},u_{i_{j-1}}))$ by
        checking the condition $\ell_{i_j}\geq t_k$ which means that the intersection $[\ell_{i_j},t_k)$ is empty (Line 4).
          After finding the maximal prefix with non-empty intersection $[\ell_{i_{j-1}},t_k)$,
            the middle point of the interval is set to $s_k$ (Line 5), and repeat the same procedure for the updated $k$ and $b_k$ (Line 8).
          
The following theorem holds for Algorithm Min\_IntSet.

\begin{theorem}
  For a given set of intervals $\{[\ell_{1},u_{1}),...,[\ell_{p},u_{p}) \}$,
the set $\{s_{1},...,s_{k}\}$ outputted by Algorithm Min\_IntSet is a minimum set that intersects all the intervals $[\ell_j,u_j)$.
\end{theorem}

\begin{proof}
  We prove the theorem by mathematical induction in the number of intervals $p$.
  For $p=1$, for-sentence between Line 3 and 11 is not executed.
  At Line 12, $s_1$ is set as
\begin{equation*}
s_1 = \frac{\ell_{i_1}+u_{i_1}}{2} 
\end{equation*}
because $t_1=u_{i_1}$, and at Line 13 Min\_IntSet outputs $\{s_1\}$, which is trivially a minimum set that intersects all the interval in the given set $\{[\ell_1,u_1)\}$.

Consider the case with $p=k+1$.
When if-sentence at Line 4 never holds, $\ell_{i_j}\leq \ell_{i_p}< t_1$ holds for all $j=1,\dots,p$
and Line 8-9 ensures $t_1\leq u_{i_j}$ for all $j=1,\dots,p$.
Thus, $[\ell_{i_p},t_1)$ is contained all the intervals and the set that is composed of its middle point $s_1$ only is
trivially a minimum set that intersects all the intervals in the given set $\{[\ell_1,u_1),\dots,[\ell_p,u_p)\}$.

 When if-sentence at Line 4 holds at least once, $s_1$ is set as
\begin{equation*}
s_1 = \frac{ \ell_{ i_{j-1} }+ t_1 }{2},
\end{equation*}
and the rest for-loop is executed for $j$ from $j+1$ to $n$ given $k=2$, $t_2=u_{i_j}$,\ and $b_2=j$.
It is easy to check that $s_2,\dots,s_k$ calculated in the rest part are the same as those outputted by
\begin{equation*}
  \text{Min\_IntSet}( \{ [\ell_{i_j},u_{i_j} ),...,[\ell_{i_p},u_{i_p} )  \} ).
\end{equation*}
The condition of if-sentence at Line 4 ensures that the intersection $[\ell_{i_{j-1}},t_1)$ of $[ \ell_{i_1}, u_{i_1} ),\dots,$ $[ \ell_{i_{j-1}}, u_{i_{j-1}} )$
does not intersect the rest intervals $[\ell_{i_j},u_{i_j} ),\dots,[\ell_{i_p},u_{i_p} )$.
Thus, any minimum set that intersects all the intervals must contain at least one value that is at most $t_1$,
and any value $s_1$ in $[\ell_{i_{j-1}},t_1)$ can minimize the set of the rest intervals that does not contain $s_1$.
Since $s_1$ is set to the middle point of $[\ell_{i_{j-1}},t_1)$ in Min\_IntSet, the rest set of intervals is minimized.
    The set of the rest points $s_2,\dots,s_k$ calculated by Min\_IntSet is the same as the set outputted by $Min\_IntSet( \{ [\ell_{i_j},u_{i_j} ),...,[\ell_{i_p},u_{i_p} )  \})$,
        so the minimum set intersecting all the intervals in $\{ [\ell_{i_j},u_{i_j} ),...,[\ell_{i_p},u_{i_p} )  \}$ can be obtained using inductive assumption.
            Thus, Min\_IntSet outputs a minimum set that intersects all the given intervals in the case with $p=k+1$.
\end{proof}

The time complexity of Algorithm Min\_IntSet is $O(p\log p)$ for the number of intervals $p$ due to the bottleneck of sorting.
Its space complexity is trivially $O(p)$.

\section{Algorithm for Minimizing the Number of Distinct Branching Conditions}

\begin{algorithm}[tbh]
	\caption{Min\_DBN}\label{Min_DBN}
\begin{algorithmic}[1]
  \REQUIRE
  \begin{minipage}[t]{10cm}
  $\{\mathbf{x}_{1},...,\mathbf{x}_{n}\}$ : Set of feature vectors\\
    $\{T_{1},...,T_{m}\}$ : decision forest\\
    $\sigma$ : path-changeable rate ($0\leq \sigma<1$)
    \end{minipage}
\STATE $L_{i} \leftarrow \emptyset$ for $i=1,\dots,d$
\FOR {$j=1$ to $m$}
\FOR{ \textbf{each} branching node $N_{j,h}$ in $T_j$} 
\STATE $(i,\theta) \leftarrow$ branching condition attached to $N_{j,h}$
\STATE $[\ell_{j,h},u_{j,h}) \leftarrow$ \begin{minipage}[t]{10cm}the range of values that $\theta$ can take without changing\\ the paths of $\mathbf{x}_{1},..,\mathbf{x}_{n}$ passing through $N_{j,h}$ in $T_j$\end{minipage}\label{alg:make-interval}
\STATE $L_i\leftarrow L_i\cup \{[\ell_{j,h},u_{j,h})\}$
  \ENDFOR
  \ENDFOR
\FOR {$i =1$ to $d$}
   \STATE $\{s_{1},..,s_{k}\},\{I_{1},...,I_{k}\} \leftarrow \text{Min\_IntSet}(L_{i})$
 \FOR {$g=1$ to $k$}
\FOR { \textbf{each} $(j,h) \in I_{g}$}
\STATE Replace the branching condition $(i,\theta)$ attached to node $N_{j,h}$ with $(i,s_g)$. 
\ENDFOR\ENDFOR
\ENDFOR
\end{algorithmic}
\end{algorithm}

Min\_DBN (Algorithm~\ref{Min_DBN}) is an algorithm for the problem of minimizing the number of distinct branching conditions in a decision forest.
The algorithm uses Algorithm Min\_IntSet for each feature $i=1,\dots,d$ to find a minimum set of branching thresholds that can share the same value without changing the paths of given feature vectors passing through each node
of each tree in a given decision forest.

For the branching condition $(i,\theta)$ attached to each branching node $N_{j,h}$ of decision tree $T_j$ in a given decision forest $\{T_1,\dots,T_m\}$,
Algorithm Min\_DBN calculates the range of values $[\ell_{j,h},u_{j,h})$ that $\theta$ can take without changing the paths of given feature vectors
$\mathbf{x}_1,\dots,\mathbf{x}_n$ passing through $N_{j,h}$ in $T_j$ (Line 2-8), and adds the range (interval) to $L_i$, which is initially set to $\emptyset$ (Line 1). 
  Then, by running Min\_IntSet for each $L_i$ ($i=1,\dots,d$), Min\_DBN obtains its output $\{s_1,...,s_k\}$ (Line 10),
  and the branching condition $(i,\theta)$ of node $N_{j,h}$ with $s_g\in [\ell_{j,h},u_{j,h})$ is replaced with $(i,s_g)$ (Line 11-15).

\iffalse
    Note that, for node $N_{j,h}$ with branching condition $(i,\theta)$ in decision tree $T_j$, the interval $[\ell_{j,h},u_{j,h})$ in which threshold $\theta$
      can take a value without changing more than $100\sigma$\% of paths of feature vectors $\mathbf{x}_1,\dots,\mathbf{x}_n$ passing through $N_{j,h}$ in $T_j$, is expressed as
\begin{eqnarray*}
  \ell_{j,h} & = & \inf_{\ell} \{\ell\mid |\{\mathbf{x}_f\in X_{j,h}\mid \ell < x_{f,i}\leq \theta\}|\leq \sigma |X_{j,h}|\} \text{ and }\\
  u_{j,h} & = & \sup_{u} \{u\mid |\{\mathbf{x}_f\in X_{j,h}\mid \theta < x_{f,i}\leq u\}|\leq \sigma |X_{j,h}|\},
\end{eqnarray*}
where 
\begin{equation*}
  X_{j,h}=\{\mathbf{x}_f| \text{The path of } \mathbf{x}_f \text{ in } T_j \text{ passes through node } N_{j,h}\},
\end{equation*}
and $|S|$ for set $S$ denotes the number of elements in $S$.
\fi

Let us analyze the time and space complexities of Min\_DBN.
\iffalse
Let $N$ denote the number of nodes in a given decision forest.
For each branching node $N_{j,h}$, Min\_DBN needs $O(|X_{j,h}|\log (\sigma |X_{j,h}|+1))\leq O(n\log (\sigma n+1))$ time for calculating $\ell_{j,h}$ and $u_{j,h}$
using size-$(\sigma |X_{j,h}|+1)$ heap. Min\_IntSet($L_i$) for all $i=1,\dots,d$ totally consumes at most $O(N\log N)$ time. 
Considering that $O(d)$ time is needed additionally, time complexity of Min\_DBN is $O(N(n\log (\sigma n+1) + \log N)+d)$.
Space complexity of Min\_DBN is $O(dn+N)$ because space linear in the sizes of given feature vectors and decision forest are enough to run Min\_DBN.
\fi
Let $N$ denote the number of branching nodes in a given decision forest.
For each branching node $N_{j,h}$, Min\_DBN needs $O(n)$ time for calculating $\ell_{j,h}$ and $u_{j,h}$.
Min\_IntSet($L_i$) for all $i=1,\dots,d$ totally consumes at most $O(N\log N)$ time.
Considering that $O(d)$ time is needed additionally, time complexity of Min\_DBN is $O(N(n + \log N)+d)$.
The space complexity of Min\_DBN is $O(dn+N)$ because space linear in the sizes of given feature vectors and decision forest is enough to run Min\_DBN.

\section{Further Sharing of Branching Conditions}

There is a possibility that our proposed algorithm cannot reduce the number of distinct branching conditions enough for some target applications.
As modifications for such cases, we propose the following two methods.

\subsection{Using the Bootstrap Training Samples per Tree}

For any given feature vector, the decision path in any component tree must not be changed in Problem~\ref{prob:DBC}. 
This constraint becomes weak if given feature vectors are reduced, which may make the smaller NDC for the problem possible. Each tree in a random forest and an extremely randomized tree is learned using a different bootstrap sample, thus training data for each component tree is different. Considering that branching conditions in each tree are decided depending on the given bootstrap sample,
it is natural to use the bootstrap training sample for each tree as given feature vectors in our NDC minimization problem.
Therefore, the following problem modification is expected to be effective for promoting branching condition sharing by giving the bootstrap training samples per tree.

\begin{problem}[NDC minimization problem with feature vectors per tree]\label{prob:DBCptfv}
  For a given tree ensemble $\{ T_{1},...T_{m} \}$, given family of feature vector sets $\{\mathbf{x}^1_{1},...,\mathbf{x}^1_{n_1} \},\dots,\{\mathbf{x}^m_{1},...,\mathbf{x}^m_{n_m} \}$, minimize the number of distinct branching conditions $(i,\theta)$ by changing the values of some $\theta$ without changing feature vector $\mathbf{x}^j_k$'s paths passing through each internal node of each decision tree $T_j$ for $k=1,\dots,n_j$ and $j=1,\dots,m$.  
\end{problem}

\subsection{Allowing Changed Decision Path}

In Problem~\ref{prob:DBC}, branching condition sharing must be done without changing any given feature vector's path.
By loosening this condition and allowing small ratio path changes of feature vectors, each interval in which the threshold can take values
is widened, as a result, further reduction of the number of distinct branching conditions can be realized.
Problem~\ref{prob:DBC} is extended to the following problem.

\begin{problem}[Path-changeable-rate version of NDC minimization problem]
  For a given decision forest $\{ T_{1},...T_{m} \}$, a given set of feature vectors $\{\mathbf{x}_{1},...,\mathbf{x}_{n} \}$ and a given path-changeable rate $0\leq \sigma<1$, minimize the number of distinct branching conditions $(i,\theta)$ by changing the values of some $\theta$ without changing more than $100\sigma$\% of feature vectors' paths passing through each node of each decision tree $T_j$ $(j=1,\dots,m)$.  
\end{problem}

This extension can be easily incorporated into Algorithm Min\_DBN by changing Line~\ref{alg:make-interval} as
\begin{description}
\item[5':] $[\ell_{j,h},u_{j,h}) \leftarrow$ \begin{minipage}[t]{10cm}the range of values that $\theta$ can take without changing\\ more than $100\sigma$\% of paths of $\mathbf{x}_{1},..,\mathbf{x}_{n}$ passing through $N_{j,h}$ in $T_j$\end{minipage}
\end{description}

      Note that, for node $N_{j,h}$ with branching condition $(i,\theta)$ in decision tree $T_j$, the interval $[\ell_{j,h},u_{j,h})$ in which threshold $\theta$
      can take a value without changing more than $100\sigma$\% of paths of feature vectors $\mathbf{x}_1,\dots,\mathbf{x}_n$ passing through $N_{j,h}$ in $T_j$, is expressed as
\begin{eqnarray*}
  \ell_{j,h} & = & \inf_{\ell} \{\ell\mid |\{\mathbf{x}_f\in X_{j,h}\mid \ell < x_{f,i}\leq \theta\}|\leq \sigma |X_{j,h}|\} \text{ and }\\
  u_{j,h} & = & \sup_{u} \{u\mid |\{\mathbf{x}_f\in X_{j,h}\mid \theta < x_{f,i}\leq u\}|\leq \sigma |X_{j,h}|\},
\end{eqnarray*}
where 
\begin{equation*}
  X_{j,h}=\{\mathbf{x}_f| \text{The path of } \mathbf{x}_f \text{ in } T_j \text{ passes through node } N_{j,h}\},
\end{equation*}
and $|S|$ for set $S$ denotes the number of elements in $S$.

For each branching node $N_{j,h}$, the extended Min\_DBN needs $O(|X_{j,h}|(\log (\sigma |X_{j,h}|+1)+1))\leq O(n(\log (\sigma n+1)+1))$ time for calculating $\ell_{j,h}$ and $u_{j,h}$
using size-$(\sigma |X_{j,h}|+1)$ heap.
Thus, time complexity of the extended version is $O(N(n(\log (\sigma n+1)+1) + \log N)+d)$.
Space complexity of the extended Min\_DBN is $O(dn+N)$, which is the same as that of the original Min\_DBN.
  
  \subsection{Allowing Non-Intersecting Intervals}

  In Problem~\ref{prob:minintset}, the solution set must intersect all the given intervals.
  By allowing that the set does not intersect some of them, smaller-sized sets can be the solution sets.
  We extend Problem~\ref{prob:minintset} to the following problem.

\begin{problem}[Exception-allowable minimum intersecting set problem]\label{prob:minintsetwexc}
For a given set of intervals $\{[\ell_{1},u_{1}),\dots,[\ell_{p},u_{p}) \}$ and a given non-negative integer $c$, find a minimum set that intersects all the intervals $[\ell_j,u_j)$ ($j=1,\dots,p$) except at most $c$ intervals.
\end{problem}

Assume that $\ell_1\leq \ell_2\leq\cdots\leq\ell_p$.
Let $D(i,j)$ ($i=1,\dots,p+1$, $j=0,\dots,c$) denote the size $k$ of the minimum set of real numbers $\{s_1,\dots,s_k\}$ that intersects all the intervals $[\ell_h,u_h)$ for $h=i,\dots,p$ except at most $j$ intervals.
  Let $\alpha_i(h)$ ($i=1,\dots,p$, $h=0,\dots,p-i$) denote the index $m$ for which $u_m$ is the $(h+1)$th smallest value among  $u_i,\dots,u_p$ and 
  let $\beta(j)$ ($j=1,\dots,p$) denote the minimum index $m$ with $\ell_m\geq u_j$ if such $m$ exists and $p+1$ otherwise.
We also define a subsequence $h_1,\dots,h_{d_{i,j}}$ of $0,\dots,j$ which is composed of the elements $h\in \{0,\dots, j\}$ with $h=0$ or $\beta(\alpha_i(h-1))<\beta(\alpha_i(h))$, where $d_{i,j}$ is the number of such elements. Note that $h_1=0$ always holds.

\begin{algorithm}[tb]
\caption{Min\_IntSet\_wExc}\label{Min_IntSet_wExc}
\begin{algorithmic}[1]
  \REQUIRE
\begin{minipage}[t]{14cm}
  $c$ : number of allowable exceptions\\
  $[\ell_{1},u_{1}),\dots,[\ell_p,u_p)$ : $p(>c)$ intervals with $\ell_i\leq \ell_{i+1}$ ($i=1,\dots,p-1$)
\end{minipage}
\ENSURE \begin{minipage}[t]{13.4cm}
  $ \{ s_{1},...,s_{k} \}$ : Minimum set satisfying $\{ s_{1},\dots,s_{k} \}\cap[\ell_{i},u_{i}) \neq \emptyset$\\
    \hspace*{\fill}$(i\in\{1,\dots,p\}\setminus E)$ with $|E|\leq c$\\
    $\{ I_{1},..I_{k} \}$ : $I_{j}=\{i \in I | j=\arg\min_{j'} d(s_{j'},[\ell_i,u_{i}))\} \  (j=1,...,k)$,\\
      \hspace*{\fill}where $d(q,S)=\inf_{x\in S}|q-x|$
      \end{minipage}
\STATE $i_1,\dots,i_p\gets$ index list sorted by $u_i$ so as to satisfy $u_{i_1}\leq u_{i_2}\leq \dots\leq u_{i_p}$
\STATE $k\gets 1$
\FOR {$j=1$ to $p$}
\STATE  \textbf{while} $u_{i_j}<\ell_k$ \textbf{do} $k\gets k+1$ \COMMENT{$k=\min \{k'\mid u_{i_j}\geq \ell_{k'}\}$}
\STATE $\beta(i_j)\gets k$
\ENDFOR
\STATE $\alpha(0)\gets p+1$, $u_{p+1}\gets \infty$, $d\gets 0$  \COMMENT{$u_{p+1}$: dummy}
\STATE Min\_IntSet\_Rec(i) for $i=p+1,\dots,1$
\STATE $j\gets c$
\STATE \textbf{while} $j>0$ and $D(1,j)=D(1,j-1)$ \textbf{do} $j\gets j-1$
\STATE $i'\gets 1$, $(i,h)\gets N(i,j)$, $j\gets j-h$
\WHILE {$i\leq p$}
\STATE $k\gets k+1$
\STATE $s_k\gets (\ell_{\beta(i)-1}+u_i)/2$
\IF {$k=1$}
\STATE $I_k\gets \{1,\dots,i-1\}$
\ELSE
\FOR {$g=i'$ to $i-1$}
\IF {$\ell_g-s_{k-1}<s_k-u_g$} 
\STATE $I_{k-1}\gets I_{k-1}\cup \{g\}$
\ELSE
\STATE $I_k\gets I_k \cup \{g\}$
\ENDIF
\ENDFOR
\ENDIF
\STATE $i'\gets i$, $(i,h)\gets N(i,j)$, $j\gets j-h$
\ENDWHILE
\STATE $k\gets k+1$, $I_k\gets \{i',\dots,n\}$
\RETURN $\{ s_{1},...,s_{k} \}, \{ I_{1},..,I_{k} \} $
\end{algorithmic}
\end{algorithm}

\begin{algorithm}[tb]
\caption{Min\_IntSet\_Rec}\label{Min_IntSet_Rec}
\begin{algorithmic}[1]
\REQUIRE $i$ : interval index (the interval $[\ell,u)$ with the $i$th smallest $\ell$) 
\STATE $j\gets c$
\WHILE {$j>=p+1-i$}
\STATE $D(i,j)\gets 0$, $N(i,j)\gets (p+1,0)$;
\STATE $j\gets j-1$
\ENDWHILE
\STATE \textbf{if} $j<0$ \textbf{then} \textbf{return}
\STATE $h\gets 0$, $k\gets 1$
\WHILE {$h\leq j$}
\STATE \textbf{if} $u_{\alpha(h)}>u_i$ \textbf{then} \textbf{break}
\STATE $D(i,h)\gets D(i+1,h)$, $N(i,h)\gets N(i+1,h)$
\STATE \textbf{if} $h_k=h$ \textbf{then} $k\gets k+1$
\ENDWHILE
\STATE $\alpha(m+1)\gets \alpha(m)$ for $m=j,\dots,h$
\STATE $\alpha(h)\gets i$
\STATE $h_m \gets h_m+1$ for $m=k,\dots,d$
\STATE \textbf{if} $d>0$ and $h_d>c$ \textbf{then}  $d\gets d-1$
\IF {$h=0$ or $\beta(i)>\beta(\alpha(h-1))$)}
\IF {$\beta(i)<\beta(\alpha(h_k))$}
\STATE $h_{m+1}\gets h_m$ for $m=d,\dots,k$
\STATE $d\gets d+1$
\ENDIF
\STATE $h_k\gets h$
\ENDIF
\WHILE {$h\leq j$}
\STATE  $h_*\gets 0$, $i_*\gets \beta(\alpha(0))$
\FOR { $k=1$ to $d$}
\STATE \textbf{if} $h_k>h$ \textbf{then} \textbf{break}
\IF {$D(\beta(\alpha(h_k)),h-h_k)<D(i_*,h-h_*)$}
\STATE $h_*\gets h_k$, $i_*\gets \beta(\alpha(h_k))$
\ENDIF
\ENDFOR
\STATE $D(i,h)\gets D(i_*,h-h_*)+1$, $N(i,h)\gets (i_*,h-h_*)$
\STATE $h\gets h+1$
\ENDWHILE
\RETURN
\end{algorithmic}
\end{algorithm}

 \begin{theorem}
  $D(i,j)$ for $i=1,\dots,p+1$ and $j=0,\dots,c$ satisfies the following recursive formula:
  \begin{align}
  D(i,j)
    =&\begin{cases}
      0 & (i+j\geq p+1)\\
  \displaystyle\min_{k\in \{1,\dots,d_{i,j}\}} D(\beta(\alpha_i(h_k)),j-h_k)+1& (i+j<p+1).
  \end{cases} \label{eq:D}
  \end{align}
  \end{theorem}
  \begin{proof}
    In the case with $i+j\geq p+1$, Eq. (\ref{eq:D}) holds trivially because the number of intervals $p+1-i$ is at most $j$ in this case.

    In the case with $i+j< p+1$, $D(i,j)>0$ holds because the number of intervals $p+1-i$ is more than $j$.
Let $S$ be the minimum-sized real number set that intersects all the intervals $[\ell_h,u_h)$ for $h=i,\dots,p$ except at most $j$ intervals, and define $s_1=\min S$.
Then, $s_1< u_{\alpha_i(j)}$ must hold because at least $j+1$ intervals $[\ell_i,u_i)$ for $i=\alpha_i(0),\dots,\alpha_i(j)$ do not intersect $S$ otherwise.
Thus, $s_1$ must be included in one of $d_{i,j}$ regions $[\ell_1,u_{\alpha_i(h_1)}),[u_{\alpha_i(h_1)},u_{\alpha_i(h_2)}),\dots,[u_{\alpha_i(h_{d_{i,j}-1})},u_{\alpha_i(h_{d_{i,j}})})$.
Divide intervals $[\ell_i,u_i),\dots,[\ell_p,u_p)$ into $[\ell_i,u_i),\dots,[\ell_{m-1},u_{m-1})$ and $[\ell_m,u_m),\dots,[\ell_p,u_p)$ by $m=\min \{h\mid \ell_h> s_1\}$.
Note that $S$ can intersect the prefix interval list $[\ell_i,u_i),\dots,[\ell_{m-1},u_{m-1})$ by $s_1$ only,
and the suffix interval list $[\ell_m,u_m),\dots,[\ell_p,u_p)$ by $S\setminus \{s_1\}$ only.
In the case with $s_1\in [\ell_1,u_{\alpha_i(h_1)})$, where $h_1=0$, we can minimize the size of $S$ by setting $s_1\in [\ell_{\beta_i(\alpha_i(0))-1},u_{\alpha_i(0)})$, for example $s_1=\frac{\ell_{\beta_i(\alpha_i(0))-1}+u_{\alpha_i(0)}}{2}$,
because this $s_1$ minimize the suffix list ($[\ell_{\beta_i(\alpha_i(0))},u_{\beta_i(\alpha_i(0))}),\dots,[\ell_p,u_p)$) and also minimize the number of intervals that are in the prefix list and not intersecting with $s_1$, which results in allowing more exceptions in the suffix list.
    Therefore, the minimum size of $S$ can be represented as $D(\beta_i(\alpha_i(h_1)),j)+1$ in this case because $s_1$ intersects all the intervals in the prefix list.
    
    In the case with $s_1\in [u_{\alpha_i(h_{k-1})},u_{\alpha_i(h_k)})$ ($k=2,\dots,d_{i,j}$), set $s_1\in [\ell_{\beta_i(\alpha_i(h_k))-1},u_{\alpha_i(h_k)})$, for example $s_1=\frac{\ell_{\beta_i(\alpha_i(h_k))-1}+u_{\alpha_i(h_k)}}{2}$.
      Note that $s_1\in [u_{\alpha_i(h_{k-1})},u_{\alpha_i(h_k)})$ holds because $u_{\alpha_i(h_{k-1})}\leq \ell_{\beta_i(\alpha_i(h_{k-1}))}\leq \ell_{\beta_i(\alpha_i(h_{k}-1))}\leq \ell_{\beta_i(\alpha_i(h_{k}))-1}<u_{\alpha_i(h_k)}$.
        This $s_1$ minimize the suffix list ($[\ell_{\beta_i(\alpha_i(h_k))},u_{\beta_i(\alpha_i(h_k))}),\dots,[\ell_p,u_p)$) and also minimize the number of the intervals that are in the prefix list and not intersecting with $s_1$. Note that, in prefix list, non-intersecting intervals with this $s_1$ are  
$[\ell_{\alpha_i(0)},u_{\alpha_i(0)}),\dots,[\ell_{\alpha_i(h_k)},u_{\alpha_i(h_k)})$, and more intervals are not intersecting with $s_1$ if $s_1\in [u_{\alpha_i(h_{k-1})},\ell_{\beta_i(\alpha_i(h_k))-1})$.
        Therefore, in this case, the minimum size of $S$ can be represented as $D(\beta_i(\alpha_i(h_k)),j-h_k)+1$.
        Thus, Eq.~(\ref{eq:D}) also holds for $i+j<p+1$.
 \end{proof}

 \begin{remark}\label{remark:d(i,j)}
    $D(i,j)=D(i+1,j)$ holds if $j$ is smaller than $h$ with $\alpha_i(h)=i$
    because the smallest real value $s_1$ must be less than $\alpha_i(j)$ and $\alpha_i(j)\leq \alpha_i(h)=i$ holds in this case,
    which means any possible $s_1$ intersects $[\ell_i,u_i)$. 
 \end{remark}
  
  What we want is $D(1,c)$ and it can be calculated using Eq.~(\ref{eq:D}) by dynamic programming.
Algorithm Min\_IntSet\_wExc (Algorithm~\ref{Min_IntSet_wExc}) is an algorithm for Problem~\ref{prob:minintsetwexc} with the number of exception $c$ and an interval list $[\ell_{1},u_{1}),\dots,[\ell_p,u_p)$  ($p>c)$ sorted by $\ell_i$ in ascending order.
    The algorithm outputs the answer $\{s_1,\dots,s_k\}$ for the problem and also outputs the set $I_j$ for each $j=1,\dots,k$ whose member interval $i$'s nearest value $s_{j'}$ among $\{s_1,\dots,s_k\}$ is $s_j$, that is, $j=\arg\min_{j'} d(s_{j'},[\ell_i,u_{i}))\}$, where $d(q,S)=\inf_{x\in S}|q-x|$. In the algorithm, the values of function $\beta:\{1,\dots,p\}\rightarrow \{2,\dots,p+1\}$ are set first in Lines~1-6,
      then $D(i,\cdot)$ is calculated by calling sub-procedure Min\_IntSet\_Rec($i$) (Algorithm~\ref{Min_IntSet_Rec}) in descending order of $i$ using dynamic programming in Lines~7-8. After executing Line~8, $D(i,j)$ for all $(i,j)\in \{1,\dot,p+1\}\times \{0,\dots,c\}$ have been set and $D(1,c)$ is set to the size of the minimum set for the problem. Then, the algorithm finds the smallest exception number $j$ satisfying $D(1,j)=D(1,c)$ in Lines~9-10,
      and construct the optimal value set $\{s_1,\dots,s_k\}$ and interval index set family $\{I_1,\dots,I_k\}$ assigned to its member value
      (Lines~11-28) by tracing the optimal $D(\beta(\alpha_i(h_k)),j-h_k)$ in Recursive formula~(\ref{eq:D}) which is stored in $N(i,j)$ by executing
      Min\_IntSet\_Rec($i$).

      Min\_IntSet\_Rec($i$) calculates $D(i,\cdot)$ according to Recursive formula~(\ref{eq:D}) using dynamic programming.
      Set $D(i,j)$ to $0$ for $i+j\geq p+1$ in Lines~1-5 first, then, in Lines~7-12, set $D(i,j)$ to $D(i+1,j)$ for $j< h_{(i)}$ with $\alpha_i(h_{(i)})=i$ following the rule described in Remark~\ref{remark:d(i,j)}.
      At the beginning of the procedure, $\alpha(0),\dots,\alpha(c)$ and $h_1,\dots, h_d$ are set to $\alpha_{i-1}(0),\dots,\alpha_{i-1}(c)$ and their corresponding $h_1,\dots,h_{d_{i-1,j}}$, respectively. In Lines~13-23, $\alpha(0),\dots,\alpha(c)$ and $h_1,\dots,h_d$ are updated to $\alpha_{i}(0),\dots,\alpha_{i}(c)$ and their corresponding $h_1,\dots,h_{d_{i,j}}$ by inserting $i$ into $\alpha(0),\dots,\alpha(c)$ at the appropriate position $h$ which is searched in the while loop right before. In the last loop from Line 24 to 34, $D(i,j)$ is set to $\min_{k\in \{1,\dots,d_{i,j}\}} D(\beta(\alpha_i(h_k)),j-h_k)+1$ by Recursive formula~(\ref{eq:D}) for $j\geq h_{(i)}$.

      The time and space complexity of Algorithm Min\_IntSet\_wExc are $O(pc^2+p\log p)$ and $O(pc)$ because $O(p\log p)$ time is needed for sorting the index list by $u_i$, $D(i,\cdot)$ calculation in Min\_IntSet\_Rec($i$) consumes $O(\sum_{j=0}^c(j+1))=O(c^2)$ for each $i=p+1,\dots,1$ and Tables $D$ and $N$ need $O(pc)$ space. Thus, Problem~\ref{prob:minintsetwexc} can be solved in $0(p(c^2+\log p))$ time and $O(pc)$ space even including the sorting by $\ell_i$.

      A heuristic approximate algorithm for Problem~\ref{prob:DBC} can be made by modifying Algorithm Min\_DBN
      as follows using Min\_IntSet\_wExc instead of Min\_IntSet.
      
\begin{algorithmic}[1]
\setcounter{ALC@line}{8}
\FOR {$i=1$ to $d$}
\algsetup{
  linenodelimiter = {'-1:}
}
\STATE $[\ell_{j_1,h_1},u_{j_1,h_2}),\dots,[\ell_{j_{n_i},h_{n_i}},u_{j_{n_i},h_{n_i}})\gets$ $L_i$'s elements sorted by $\ell_{j,h}$ in ascending order
\setcounter{ALC@line}{9}
\algsetup{
  linenodelimiter = {'-2:}
}
\STATE $\{s_1,\dots,s_k\},\{I_1,\dots,I_k\}\gets$ Min\_IntSet\_wExc($[\ell_{j_1,h_1},u_{j_1,h_2}),\dots,[\ell_{j_{n_i},h_{n_i}},u_{j_{n_i},h_{n_i}})$)
\algsetup{
  linenodelimiter = {:}
}    
\FOR {$g=1$ to $k$}
\algsetup{
  linenodelimiter = {':}
}    
\FOR {\textbf{each} $f\in I_g$}
\STATE Replace the branching condition $(i,\theta)$ attached to node $N_{j_f,h_f}$ with $(i,s_g)$
\algsetup{
  linenodelimiter = {:}
}
\ENDFOR
\ENDFOR
\ENDFOR
\end{algorithmic}

Let $N$ denote the number of branching nodes in a given decision forest.
Since Min\_IntSet\_wExc ($[\ell_{j_1,h_1},u_{j_1,h_2}),\dots,[\ell_{j_{n_i},h_{n_i}},u_{j_{n_i},h_{n_i}})$) for all $i=1,\dots,d$ totally consumes at most $O(N(c^2+\log N))$ time, time complexity of this version of Min\_DBN is $O(N(n + c^2+\log N)+d)$.
Space complexity of this Min\_DBN is $O(dn+Nc)$ considering the space used by Tables $D$ and $N$.

\section{Experiments}

We conducted experiments to check the effectiveness of branching condition sharing
using datasets in the UCI machine learning repository \citep{Dua:2017}.

\subsection{Settings}

\begin{table}[tb]
  \caption{Datasets used in our experiments. Notation $C_\ell$ means that the number of distinct class labels is $\ell$.}\label{tbl:dataset}
  {\scriptsize
  \begin{tabular}{|l|r|r|c|p{8.2cm}|}
\hline
dataset & \multicolumn{1}{c|}{$n$} & \multicolumn{1}{c|}{$d$} & $Y$ & reference in \citep{Dua:2017}\\
\hline
iris & 150 & 4 & $C_3$ & Iris\\
parkinsons & 195 & 22 & $C_2$ & Parkinsons \citep{parkinson}\\
breast cancer & 569 & 30 & $C_2$ & Breast Cancer Wisconsin (Diagnostic)\\
blood & 748 & 4 & $C_2$ & Blood Transfusion Service Center \citep{blood}\\
RNA-Seq PANCAN & 801 & 20531 & $C_5$ & gene expression cancer RNA-Seq \citep{PANCAN}\\
winequality red & 1599 & 11 & $C_{11}$ & Wine Quality \citep{wine}\\
winequality white & 4898 & 11 & $C_{11}$ & Wine Quality \citep{wine}\\
waveform & 5000 & 40 & $C_3$ & Waveform Database Generator (Version 2)\\
robot & 5456 & 24 & $C_4$ & Wall-Following Robot Navigation\\
musk & 6598 & 166 & $C_2$ & Musk (Version 2)\\
epileptic seizure & 11500 & 178 & $C_5$ & Epileptic Seizure Recognition \citep{EpilepticSeizureRecognition}\\
magic & 19020 & 10  & $C_2$ & MAGIC Gamma Telescope\\
hepmass & 7000000 & 28 & $C_2$ & HEPMASS (train)\\
\hline
Real estate & 414 & 7 & $\mathbb{R}$ & Real estate valuation data set \citep{RealEstate}\\
airfoil & 1502 & 5 & $\mathbb{R}$ & Airfoil Self-Noise Data Set\\
blogdata & 60021 & 280 & $\mathbb{R}$ & BlogFeedback Data Set \citep{blogdata}\\
Adelaide & 71999 & 48 & $\mathbb{R}$ & Wave Energy Converters Data Set\\
Perth & 72000 & 48 & $\mathbb{R}$ & Wave Energy Converters Data Set\\
Sydney & 72000 & 48 & $\mathbb{R}$ & Wave Energy Converters Data Set\\
Tasmania & 72000 & 48 & $\mathbb{R}$ & Wave Energy Converters Data Set\\
Year Predict MSD & 515344 & 90 & $\mathbb{R}$ & YearPredictionMSD Data Set\\
\hline
  \end{tabular}
}
\end{table}

We used 21 numerical-feature datasets registered in the UCI machine learning repository \citep{Dua:2017},
whose number of instances $n$, number of features $d$, and label space $Y$ are shown in Table~\ref{tbl:dataset}.
Among them, the task for the first 13 datasets is classification, and that for the rest 8 datasets is regression. 
In the table, datasets are sorted in the order of the number of instances for each task.
The largest one is hepmass dataset which has 7 million instances.
The dataset with the largest number of features is RNA-Seq PANCAN whose number of features is more than 20 thousand.
Note that the number of features is larger than the number of instances only for this dataset.
The number of distinct class labels is not so large for all the classification datasets we used,
and winequality datasets have the largest number of distinct labels (11 distinct labels).

Tree ensemble classifiers and regressors used in the experiments are random forest \citep{Breiman2001},
extremely randomized trees \citep{GEW2006}, AdaBoost \citep{FREUND1997119}, and gradient boosting \citep{Friedman00}.
The used implemented classifiers and regressors of those four tree ensemble learners are
RandomForestClassifier, RandomForestRegressor, ExtraTreesClassifier(bootstrap=True), ExtraTreesRegressor(bootstrap=True),
AdaBoostClassifier(base\_estimator=DecisionTreeClassifier(random\_state=0), AdaBoostRegressor(base\_estimator=DecisionTreeRegressor(random\_state=0)), GradientBoostingClassifier and GradientBoostingRegressor of Scikit-learn version 1.1.dev0 \citep{scikit-learn}.
The parameters of the classifiers/regressors are set to defaults except for the number of trees\footnote{Note that boosting learners stop when the learned current classifiers/regressors fit the training data even if the number of iterations does not reach n\_estimators.}  (n\_estimators), the number of jobs to run in parallel (n\_jobs) and the seed used by the random number generator (random\_state): $\text{n\_estimators}=100$, $\text{n\_jobs}=-1$ (which means the same as the number of processors) and $\text{random\_state}=0$.
Note that parameter random\_state is fixed in order to ensure that the same classifier/regressor is generated for the same training dataset. Also note that the number of randomly selected features used for branching conditions of each component tree in random forests is set to $\sqrt{d}$ as the default value.

All the experiments are done by 5-fold cross-validation using the function sklearn.model\_selection.
KFold(n\_splits=5, shuffle = True, random\_state=0), where random\_state is also fixed for reproducibility.
To classifiers/regressors learned by tree ensemble learners using training data, Min\_DBN is applied with the training data as feature vectors ignoring their labels. 
Performance is evaluated using test data by prediction accuracy and the number of distinct branching conditions of the learned classifiers/regressors,
where regressor accuracy is measured using the coefficient of determination $R^2(\{(y,\hat{y})\})=1-\sum(y-\hat{y})^2/\sum(y-\bar{y})^2$ for real label $y$, predicted label $\hat{y}$ and average label $\bar{y}$.

\subsection{Results}

\subsubsection{Effectiveness for the four tree ensemble learners}

\begin{table}[tb]
\caption{Performance of branching condition sharing for tree ensemble classifiers/regressors}\label{tbl:perf-four-types}
{\scriptsize
\begin{tabular}{@{}l@{\ }rrrr|@{\ }l@{\ }|l@{\ }rrrr@{}}
dataset & RF & ERT & ABoost & GBoost & eval & dataset & RF & ERT & ABoost & GBoost \\
\hline
\multirow{5}{*}{\begin{minipage}{1cm}iris\end{minipage}} & 1.0000&1.0000&1.0000&1.0000& TrA&\multirow{5}{*}{\begin{minipage}{1cm}magic\end{minipage}} & 1.0000&0.99999&1.0000&0.88379\\
&0.94667&0.94667&0.95333&0.96000& TeA&&0.87976&0.87429&0.81872&0.87108\\
 &103.8&1354.0&7.6&47.8& NDC &&113713.4&289735.0&1731.8&501.4\\
 &0.42004&0.074003&0.94737&0.61088& SR &&0.084643&0.028258&0.45791&0.85880\\
 &1.0211&1.0141&1.0000&1.0000& AR &&0.99928&0.99898&0.99929&0.99982\\
\hline
\multirow{5}{*}{\begin{minipage}{1cm}parkin-son\end{minipage}} & 1.0000&1.0000&1.0000&1.0000& TrA&\multirow{5}{*}{\begin{minipage}{1cm}hep-mass\end{minipage}} & 0.99999&0.99999&1.0000&0.85607\\
&0.89231&0.91795&0.85641&0.90256& TeA&&0.86937&0.86723&0.81675&0.85594\\
 &1012.6&2564.8&13.8&217.2& NDC &&34885020.4&82801422.6&449759.8&457.0\\
 &0.39897&0.25483&1.0000&0.60221& SR &&0.010160&0.0019280&0.080108&0.99737\\
 &1.0000&0.98324&1.0000&1.0000& AR &&0.99996&0.99984&0.99935&1.0000\\
\hline
\multirow{5}{*}{\begin{minipage}{1cm}breast cancer\end{minipage}} & 1.0000&1.0000&1.0000&1.0000& TrA&\multirow{5}{*}{\begin{minipage}{1cm}Real estate\end{minipage}} & 0.95285&0.95224&0.99862&0.93718\\
&0.96137&0.96842&0.91220&0.96665& TeA&&0.69149&0.69967&0.68901&0.67055\\
 &1415.4&3707.8&18.2&290.8& NDC &&5077.4&20708.8&5027.2&266.8\\
 &0.41953&0.28276&0.98901&0.65612& SR &&0.12487&0.029002&0.12735&0.66942\\
 &1.0018&0.99817&1.0000&1.0000& AR &&1.0029&1.0028&0.99757&0.99556\\
\hline
\multirow{5}{*}{\begin{minipage}{1cm}blood\end{minipage}} & 0.93683&0.93683&0.93683&0.85361& TrA&\multirow{5}{*}{\begin{minipage}{1cm}airfoil\end{minipage}} & 0.99029&0.99104&0.99947&0.88707\\
&0.74863&0.75668&0.73798&0.77937& TeA&&0.93337&0.94024&0.91629&0.85238\\
 &321.6&15786.8&267.0&137.4& NDC &&1337.0&75599.6&1489.0&123.4\\
 &0.45025&0.010046&0.48839&0.69869& SR &&0.11399&0.0019841&0.10423&0.79903\\
 &1.0000&0.99470&0.99819&0.99827& AR &&1.0002&0.99960&1.0018&0.99968\\
\hline
\multirow{5}{*}{\begin{minipage}{1cm}RNA-Seq PANCAN\end{minipage}} & 1.0000&1.0000&1.0000&1.0000& TrA&\multirow{5}{*}{\begin{minipage}{1cm}blog-data\end{minipage}} & 0.93364&0.93356&0.97597&0.68252\\
&0.99625&0.99625&0.97877&0.98002& TeA&&0.57227&0.57405&0.41710&0.56770\\
 &1924.0&3373.6&9.0&702.2& NDC &&81586.8&1141949.4&25144.4&338.0\\
 &0.95551&0.93817&1.0000&0.97835& SR &&0.090640&0.0048244&0.16200&0.86982\\
 &1.0000&1.0000&1.0000&1.0000& AR &&1.0014&0.99963&0.99156&0.99945\\
\hline
\multirow{5}{*}{\begin{minipage}{1cm}wine-quality red\end{minipage}} & 1.0000&1.0000&1.0000&0.89149& TrA&\multirow{5}{*}{\begin{minipage}{1cm}Ade-laide\end{minipage}} & 0.98747&0.98781&0.99967&0.95891\\
&0.69168&0.69731&0.64292&0.65165& TeA&&0.91242&0.91535&0.93205&0.95539\\
 &4115.4&44535.4&317.4&940.6& NDC &&3456308.8&3641198.6&2928453.0&687.2\\
 &0.21869&0.021399&0.57656&0.61748& SR &&0.017591&0.011090&0.015845&0.98370\\
 &1.0027&0.99370&1.0049&1.0048& AR &&0.99837&0.99622&0.99856&1.0000\\
\hline
\multirow{5}{*}{\begin{minipage}{1cm}wine-quality white\end{minipage}} & 1.0000&1.0000&1.0000&0.73117& TrA&\multirow{5}{*}{\begin{minipage}{1cm}Perth\end{minipage}} & 0.98520&0.98593&0.99964&0.95532\\
&0.67926&0.68048&0.59902&0.59208& TeA&&0.89634&0.90210&0.91897&0.95139\\
 &7288.6&140137.8&812.2&1186.6& NDC &&3486197.0&3629525.8&2982017.6&688.6\\
 &0.19249&0.010184&0.45900&0.63509& SR &&0.016603&0.010542&0.015146&0.98141\\
 &0.99669&1.0000&0.99932&0.99965& AR &&0.99815&0.99549&0.99811&1.0000\\
\hline
\multirow{5}{*}{\begin{minipage}{1cm}wave-form\end{minipage}} & 1.0000&1.0000&1.0000&0.95570& TrA&\multirow{5}{*}{\begin{minipage}{1cm}Sydney\end{minipage}} & 0.98797&0.98804&0.99983&0.91439\\
&0.85380&0.86260&0.73320&0.85540& TeA&&0.91627&0.91692&0.93891&0.90881\\
 &32516.8&84971.6&436.2&1221.0& NDC &&2703689.0&2937518.6&2354579.0&680.8\\
 &0.18061&0.063826&0.72536&0.80999& SR &&0.022288&0.014918&0.019310&0.97268\\
 &1.0042&1.0014&0.99891&0.99930& AR &&0.99868&0.99663&0.99855&1.0000\\
\hline
\multirow{5}{*}{\begin{minipage}{1cm}robot\end{minipage}} & 1.0000&1.0000&1.0000&1.0000& TrA&\multirow{5}{*}{\begin{minipage}{1cm}Tas-mania\end{minipage}} & 0.97651&0.97624&0.99965&0.90980\\
&0.99487&0.96793&0.99358&0.99725& TeA&&0.83493&0.83397&0.86483&0.90248\\
 &9897.2&56820.6&28.6&511.6& NDC &&3500417.0&3641163.8&3034595.4&687.2\\
 &0.31750&0.093568&0.81119&0.69977& SR &&0.016625&0.010852&0.014543&0.98254\\
 &1.0000&1.0008&1.0004&1.0006& AR &&0.99765&0.99410&0.99741&1.0000\\
\hline
\multirow{5}{*}{\begin{minipage}{1cm}musk\end{minipage}} & 1.0000&1.0000&1.0000&0.97514& TrA&\multirow{5}{*}{\begin{minipage}{1cm}Year Predict MSD\end{minipage}} & 0.90249&0.90312&0.99877&0.27836\\
&0.97514&0.97363&0.96635&0.96272& TeA&&0.30516&0.31117&0.31138&0.27334\\
 &12220.8&32575.8&135.0&440.2& NDC &&17822075.6&20674664.6&15571234.6&684.4\\
 &0.51039&0.24586&0.96000&0.91186& SR &&0.010382&0.0044004&0.010944&0.99094\\
 &0.99907&1.0005&1.0002&1.0000& AR &&0.99734&0.99282&0.99616&1.0000\\
\hline
\multirow{5}{*}{\begin{minipage}{1cm}epi-leptic seizure\end{minipage}} & 1.0000&1.0000&1.0000&0.79989&\multicolumn{6}{@{}l}{TrA: training accuracy}\\
&0.69504&0.68817&0.47339&0.61652&\multicolumn{6}{@{}l}{TeA: test accuracy}\\
 &75683.2&254271.4&2044.4&2762.0&\multicolumn{6}{@{}l}{NDC: number of distinct conditions}\\
 &0.31092&0.081560&0.66748&0.98226&\multicolumn{6}{@{}l}{SR: size ratio}\\
 &0.99812&0.99532&1.0020&1.0001&\multicolumn{6}{@{}l}{AR: accuracy ratio}\\
\hline
\end{tabular}
}
\end{table}

\begin{figure}[tb]
  \includegraphics[width=0.5\linewidth]{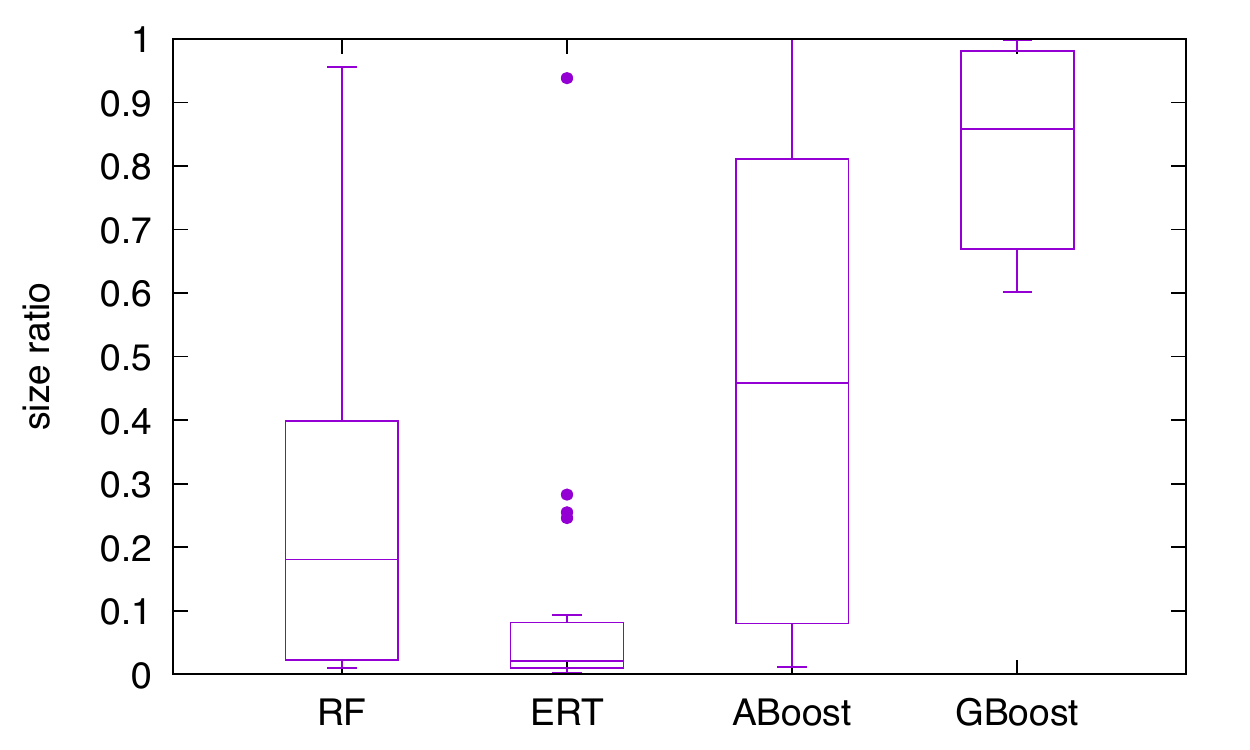}\includegraphics[width=0.5\linewidth]{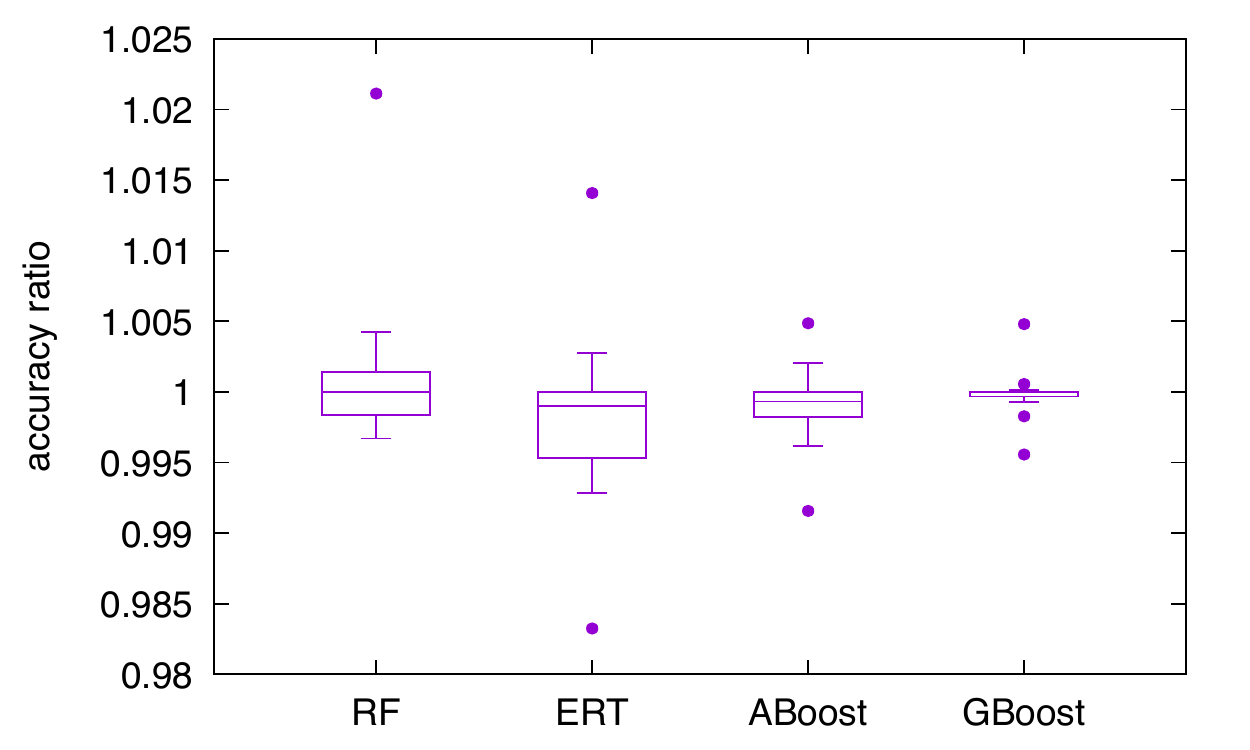}
  \caption{The box plots of size and accuracy ratios for ensemble classifiers/regressors learned from the 21 datasets by the four learners. Here, `size' means the number of distinct branching conditions, and ratios are calculated for their size and accuracy after branching condition sharing compared to their original ones. In the figures, the values outside 1.5 times the interquartile range from the first and third quartile boundaries are plotted as outliers.}\label{fig:perf-four-types}
\end{figure}

Table~\ref{tbl:perf-four-types} shows the results for applying Min\_DBN to the classifiers/regressors learned by
the four tree ensemble learners from the 21 datasets. The result for each dataset consists of 5 rows, training accuracy (TrA), test accuracy (TeA), the number of distinct branching conditions in the tree ensemble classifier/regressor learned from the training data (NDC), and the ratio of NDC (SR) and accuracy (AR) after applying Min\_DBN to the classifier/regressor.
Figure~\ref{fig:perf-four-types} shows the distribution of NDC (size) and accuracy ratios over 21 datasets
for each tree ensemble learner. You can see that our branching condition sharing reduces the number of distinct branching conditions while keeping test accuracy for all the four tree ensemble learners.
Accuracy degradation is within 1\% except for the outliers.
Our branching condition sharing is most effective for extremely randomized trees (ERTs) and next for random forests (RFs).
The size ratio medians for RF and ERT are 18.1\% and 2.1\%, respectively, and most size ratios are within 10\% for ERT.
The remarkable outlier dataset for ERT is RNA-Seq PANCAN, whose size ratio is 93.8\%, and the size ratios for the other three learners are worse.
The number of features in RNA-Seq PANCAN dataset is more than 30 times larger than the number of training instances, 
and the number of distinct features in each original component decision tree is less than 1/10 of the number of features. As a result, the number of appearing branching  conditions for each feature is small, which makes condition sharing difficult.
The sharing is also effective for boosting type learners, AdaBoost (ABoost) and gradient boosting (GBoost), whose median size ratios are 45.9\% for ABoost and 85.9\% for GBoost.
Classifiers/Regressors learned by boosting type learners have small NDCs for many datasets,
which indicates that their redundancy is small to reduce NDCs significantly.

\subsubsection{Effectiveness of using bootstrap training samples per tree}

\begin{table}[tbh]
\caption{Effect of using bootstrap training samples per tree as feature vectors for the tree}\label{tbl:bootstrap-sample}
{\scriptsize
\begin{center}
\begin{tabular}{@{}l@{\ }rr|@{\ }l@{\ }|@{\ }l@{\ }rr|@{\ }l@{\ }rr@{}}
dataset & RF & ERT & EI &dataset & RF & ERT& dataset & RF & ERT \\
\hline
\multirow{6}{*}{\begin{minipage}{1cm}iris\end{minipage}} & 0.42004&0.074003& SR(all)\textcircled{1}&\multirow{6}{*}{\begin{minipage}{1cm}wave-form\end{minipage}} & 0.18061&0.063826&\multirow{6}{*}{\begin{minipage}{1cm}airfoil\end{minipage}} & 0.11399&0.0019841\\
 &0.35838&0.071344& SR(tree)\textcircled{2}&&0.16805&0.051071&&0.11294&0.0019656\\
 &0.85321&0.96407& \textcircled{2}/\textcircled{1}&&0.93046&0.80016&&0.99081&0.99067\\
 &1.0211&1.0141& AR(all)\textcircled{3}&&1.0042&1.0014&&1.0002&0.99960\\
 &1.0141&1.0211& AR(tree)\textcircled{4}&&1.0033&0.99629&&0.99811&0.99455\\
 &0.99310&1.0069& \textcircled{4}/\textcircled{3}&&0.99907&0.99491&&0.99787&0.99495\\
\hline
\multirow{6}{*}{\begin{minipage}{1cm}parkin-son\end{minipage}} & 0.39897&0.25483& SR(all)\textcircled{1}&\multirow{6}{*}{\begin{minipage}{1cm}robot\end{minipage}} & 0.31750&0.093568&\multirow{6}{*}{\begin{minipage}{1cm}blog-data\end{minipage}} & 0.090640&0.0048244\\
 &0.32510&0.20360& SR(tree)\textcircled{2}&&0.28016&0.076229&&0.088017&0.0045007\\
 &0.81485&0.79896& \textcircled{2}/\textcircled{1}&&0.88238&0.81469&&0.97106&0.93291\\
 &1.0000&0.98324& AR(all)\textcircled{3}&&1.0000&1.0008&&1.0014&0.99963\\
 &1.0115&0.97765& AR(tree)\textcircled{4}&&1.0004&1.0002&&1.0024&0.99899\\
 &1.0115&0.99432& \textcircled{4}/\textcircled{3}&&1.0004&0.99943&&1.0010&0.99936\\
\hline
\multirow{6}{*}{\begin{minipage}{1cm}breast cancer\end{minipage}} & 0.41953&0.28276& SR(all)\textcircled{1}&\multirow{6}{*}{\begin{minipage}{1cm}musk\end{minipage}} & 0.51039&0.24586&\multirow{6}{*}{\begin{minipage}{1cm}Ade-laide\end{minipage}} & 0.017591&0.011090\\
 &0.34647&0.22995& SR(tree)\textcircled{2}&&0.48205&0.21982&&0.016816&0.0094631\\
 &0.82587&0.81324& \textcircled{2}/\textcircled{1}&&0.94446&0.89407&&0.95593&0.85328\\
 &1.0018&0.99817& AR(all)\textcircled{3}&&0.99907&1.0005&&0.99837&0.99622\\
 &1.0000&0.99636& AR(tree)\textcircled{4}&&0.99969&1.0005&&0.99791&0.99397\\
 &0.99818&0.99819& \textcircled{4}/\textcircled{3}&&1.0006&1.0000&&0.99954&0.99774\\
\hline
\multirow{6}{*}{\begin{minipage}{1cm}blood\end{minipage}} & 0.45025&0.010046& SR(all)\textcircled{1}&\multirow{6}{*}{\begin{minipage}{1cm}epi-leptic seizure\end{minipage}} & 0.31092&0.081560&\multirow{6}{*}{\begin{minipage}{1cm}Perth\end{minipage}} & 0.016603&0.010542\\
 &0.39988&0.0098943& SR(tree)\textcircled{2}&&0.29422&0.066807&&0.015871&0.0089568\\
 &0.88812&0.98487& \textcircled{2}/\textcircled{1}&&0.94627&0.81911&&0.95589&0.84961\\
 &1.0000&0.99470& AR(all)\textcircled{3}&&0.99812&0.99532&&0.99815&0.99549\\
 &0.99646&0.99823& AR(tree)\textcircled{4}&&0.99262&0.99381&&0.99757&0.99270\\
 &0.99644&1.0035& \textcircled{4}/\textcircled{3}&&0.99448&0.99848&&0.99943&0.99720\\
\hline
\multirow{6}{*}{\begin{minipage}{1cm}RNA-Seq PANCAN\end{minipage}} & 0.95551&0.93817& SR(all)\textcircled{1}&\multirow{6}{*}{\begin{minipage}{1cm}magic\end{minipage}} & 0.084643&0.028258&\multirow{6}{*}{\begin{minipage}{1cm}Sydney\end{minipage}} & 0.022288&0.014918\\
 &0.94033&0.92062& SR(tree)\textcircled{2}&&0.076688&0.022711&&0.021545&0.013568\\
 &0.98412&0.98130& \textcircled{2}/\textcircled{1}&&0.90602&0.80372&&0.96668&0.90950\\
 &1.0000&1.0000& AR(all)\textcircled{3}&&0.99928&0.99898&&0.99868&0.99663\\
 &1.0000&1.0000& AR(tree)\textcircled{4}&&1.0000&0.99814&&0.99832&0.99491\\
 &1.0000&1.0000& \textcircled{4}/\textcircled{3}&&1.0007&0.99916&&0.99964&0.99827\\
\hline
\multirow{6}{*}{\begin{minipage}{1cm}wine-quality red\end{minipage}} & 0.21869&0.021399& SR(all)\textcircled{1}&\multirow{6}{*}{\begin{minipage}{1cm}hep-mass\end{minipage}} & 0.010160&0.0019280&\multirow{6}{*}{\begin{minipage}{1cm}Tas-mania\end{minipage}} & 0.016625&0.010852\\
 &0.20994&0.019813& SR(tree)\textcircled{2}&&0.0098652&0.0015365&&0.015828&0.0092964\\
 &0.96000&0.92592& \textcircled{2}/\textcircled{1}&&0.97095&0.79695&&0.95209&0.85666\\
 &1.0027&0.99370& AR(all)\textcircled{3}&&0.99996&0.99984&&0.99765&0.99410\\
 &1.0090&0.98832& AR(tree)\textcircled{4}&&0.99991&0.99961&&0.99692&0.99030\\
 &1.0063&0.99459& \textcircled{4}/\textcircled{3}&&0.99996&0.99977&&0.99927&0.99618\\
\hline
\multirow{6}{*}{\begin{minipage}{1cm}wine-quality white\end{minipage}} & 0.19249&0.010184& SR(all)\textcircled{1}&\multirow{6}{*}{\begin{minipage}{1cm}Real estate\end{minipage}} & 0.12487&0.029002&\multirow{6}{*}{\begin{minipage}{1cm}Year Predict MSD\end{minipage}} & 0.010382&0.0044004\\
 &0.18503&0.0094007& SR(tree)\textcircled{2}&&0.11683&0.025709&&0.010117&0.0035825\\
 &0.96123&0.92307& \textcircled{2}/\textcircled{1}&&0.93565&0.88645&&0.97448&0.81413\\
 &0.99669&1.0000& AR(all)\textcircled{3}&&1.0029&1.0028&&0.99734&0.99282\\
 &1.0018&1.0003& AR(tree)\textcircled{4}&&1.0094&1.0043&&0.99547&0.98828\\
 &1.0051&1.0003& \textcircled{4}/\textcircled{3}&&1.0065&1.0015&&0.99813&0.99542\\
\hline
\end{tabular}
\end{center}
}
\end{table}

\begin{figure}[tb]
  \includegraphics[width=0.5\linewidth]{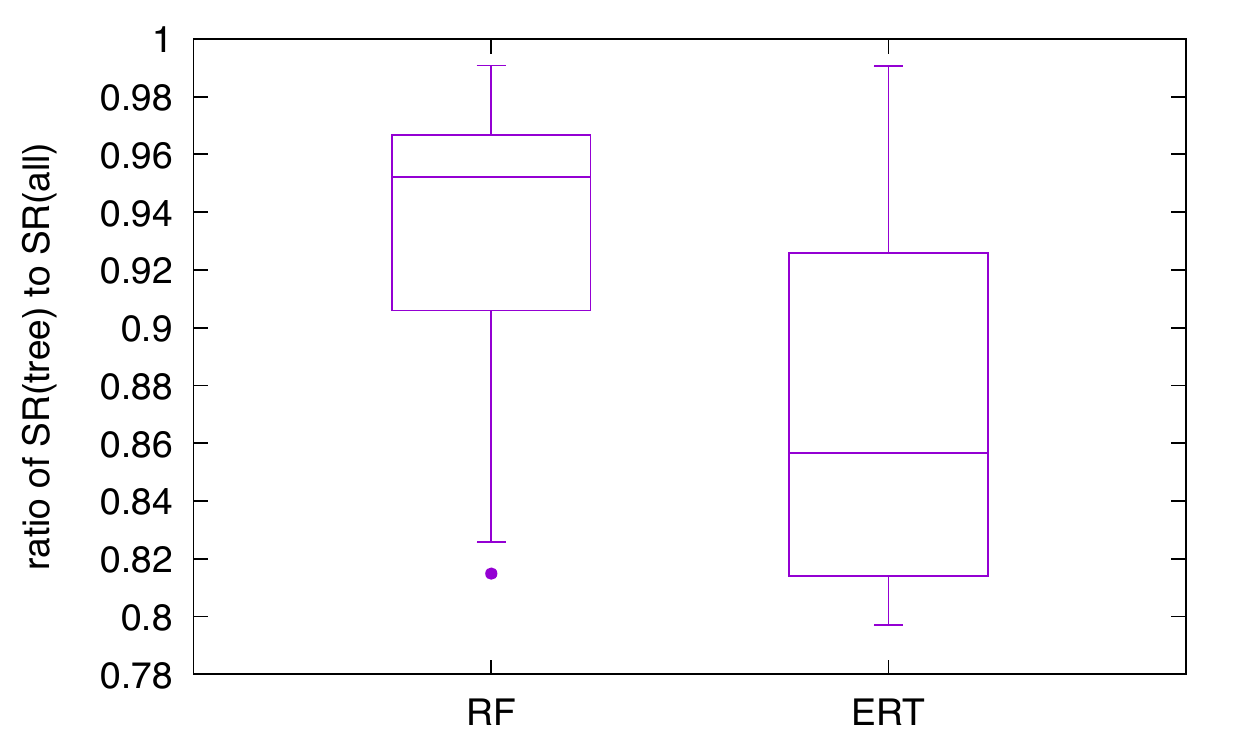}\includegraphics[width=0.5\linewidth]{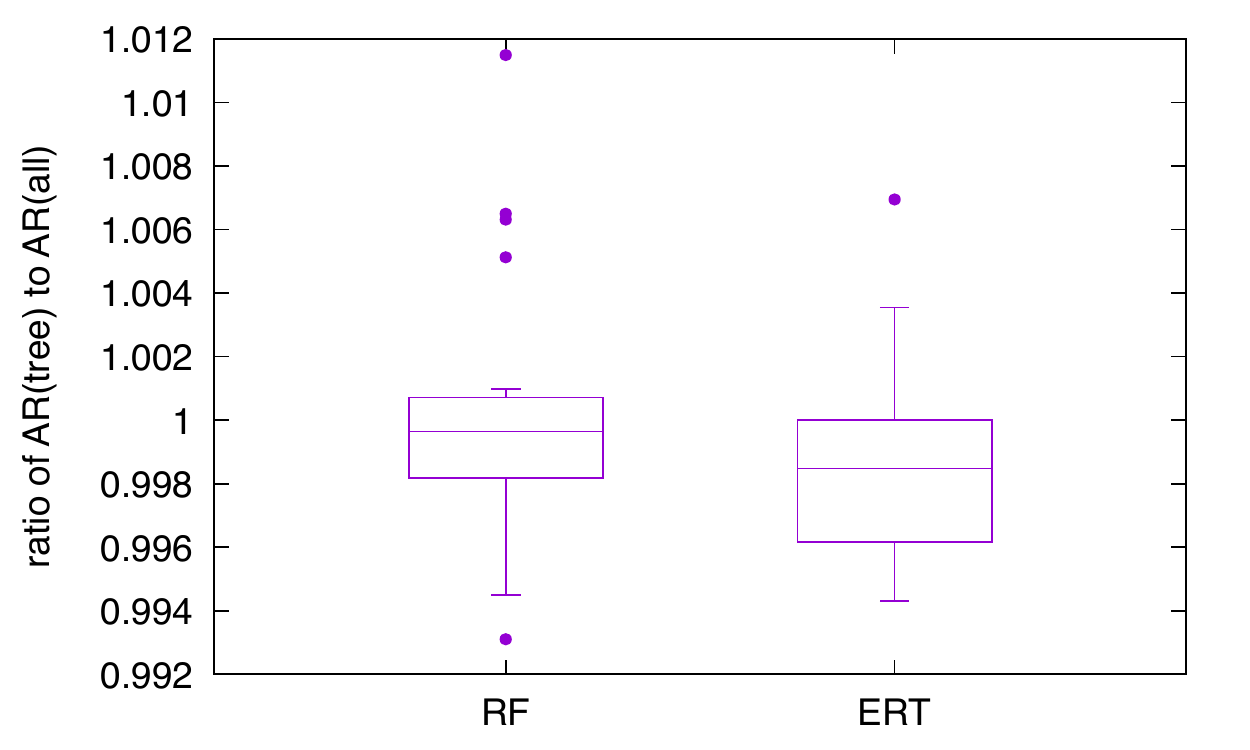}
  \caption{The box plots of ratios of SR(tree) to SR(all) and AR(tree) to AR(all)
    for classifiers/regressors learned from the 21 datasets by RF and ERT learners.
    In the figures, the values outside 1.5 times the interquartile range from the first and third quartile boundaries are plotted as outliers.}\label{fig:bootstrap-sample}
\end{figure}

Next, we check the effect of using bootstrap training samples per tree as feature vectors for the tree instead of using all the training data as feature vectors for all the trees. 
The bootstrap training samples per tree can be referred through scikit-learn by building from a little modified source.

The result is shown in Table~\ref{tbl:bootstrap-sample}.
In the table, SR(all) and AR(all) are the same as SR and AR in Table~\ref{tbl:perf-four-types},
and SR(tree) and AR(tree) are SR and AR for the case using each tree's bootstrap training samples as the feature vectors given to the tree. The ratio of SR(tree) to SR(all) and the ratio of AR(tree) to AR(all) are also shown in the third and sixth rows, respectively, in the rows for each dataset.
Figure~\ref{fig:bootstrap-sample} shows the distributions of those ratios over ensemble classifiers/regressors learned from the 21 datasets by RF and ERT learners.
The medians of such relative size ratios for RFs and ERTs are 95.2\% and 85.7\%, respectively, while also keeping the relative accuracy degradation within 1\%.

\subsubsection{Effectiveness of further sharing}

\begin{figure}[tb]
  \includegraphics[width=0.33\linewidth]{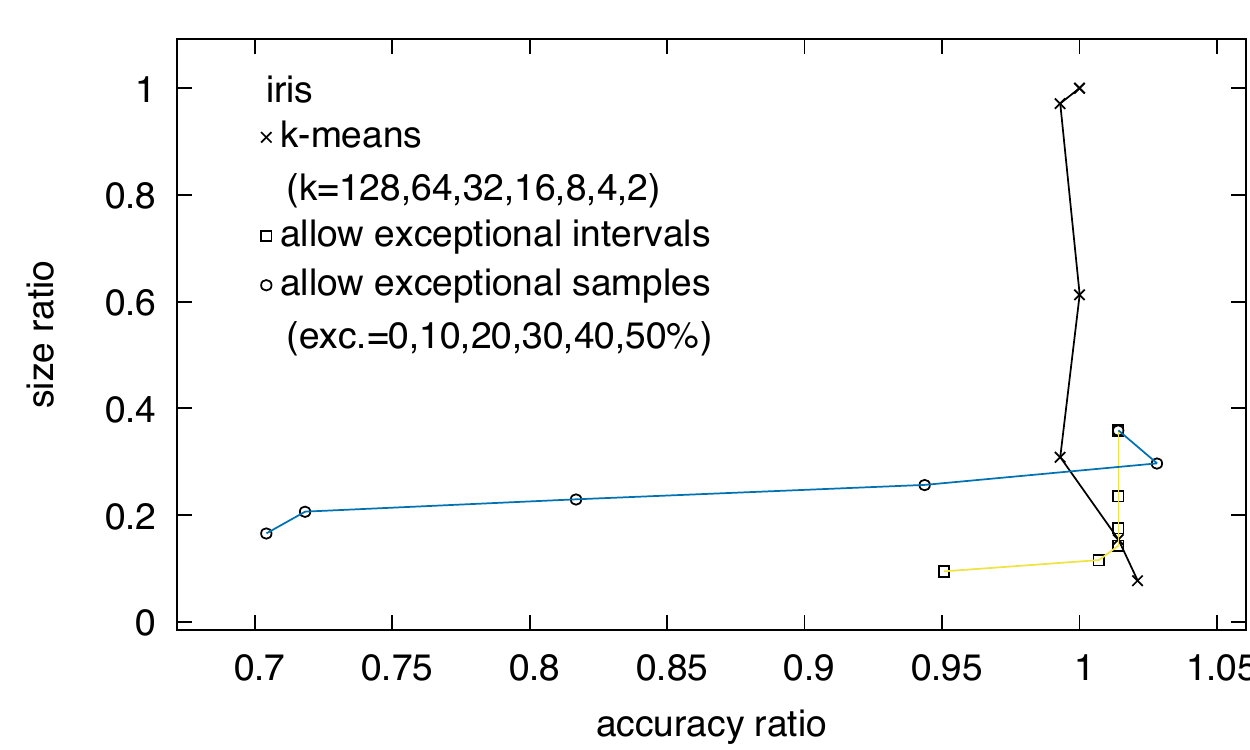}\includegraphics[width=0.33\linewidth]{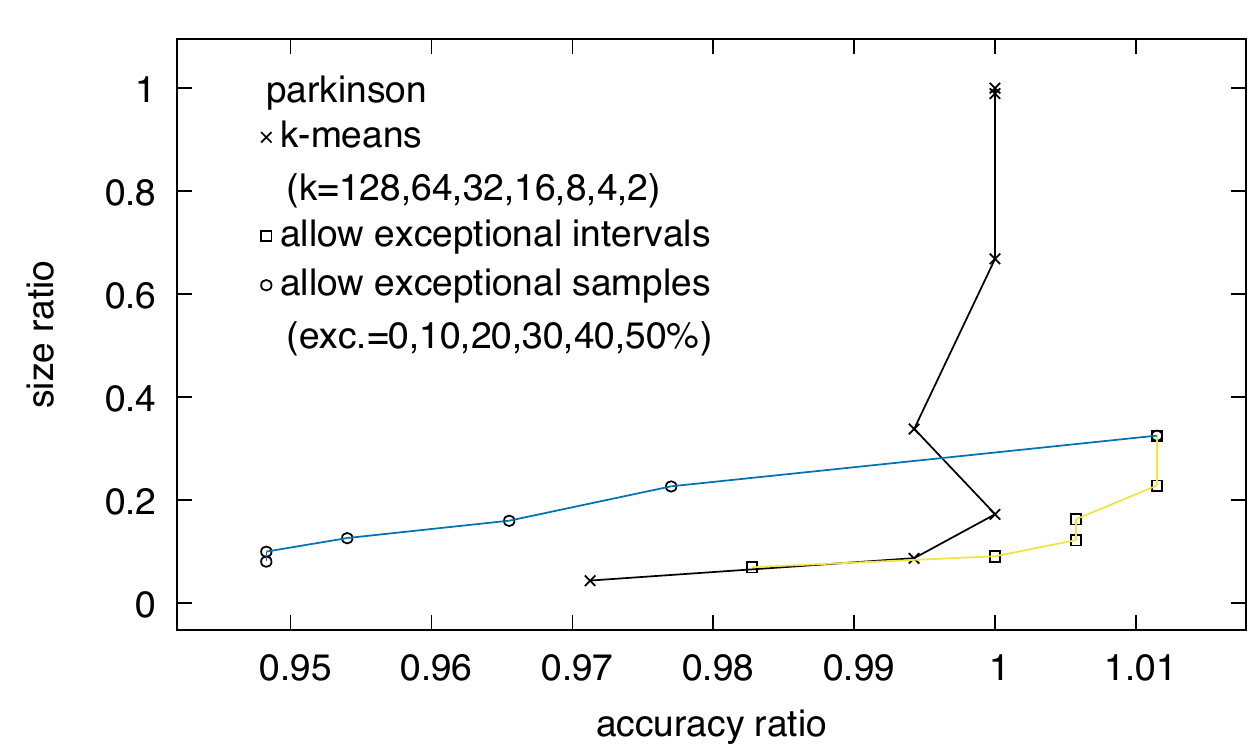}\includegraphics[width=0.33\linewidth]{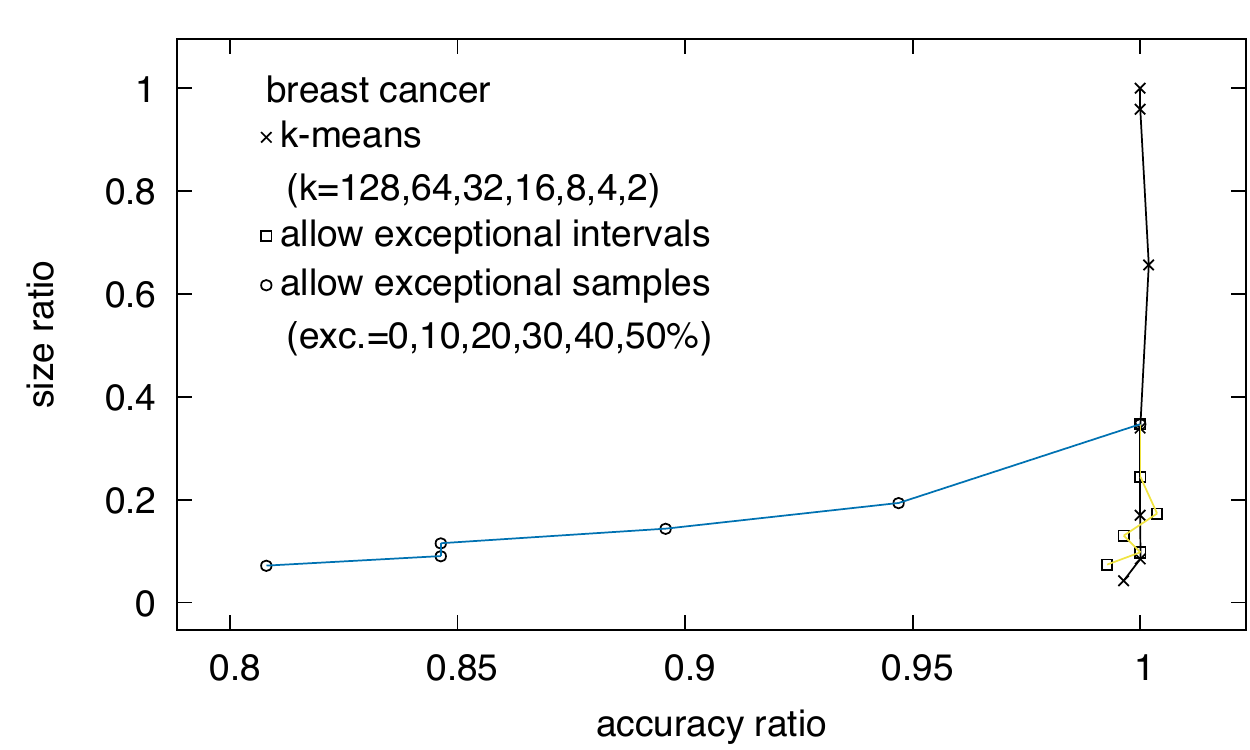}
  \includegraphics[width=0.33\linewidth]{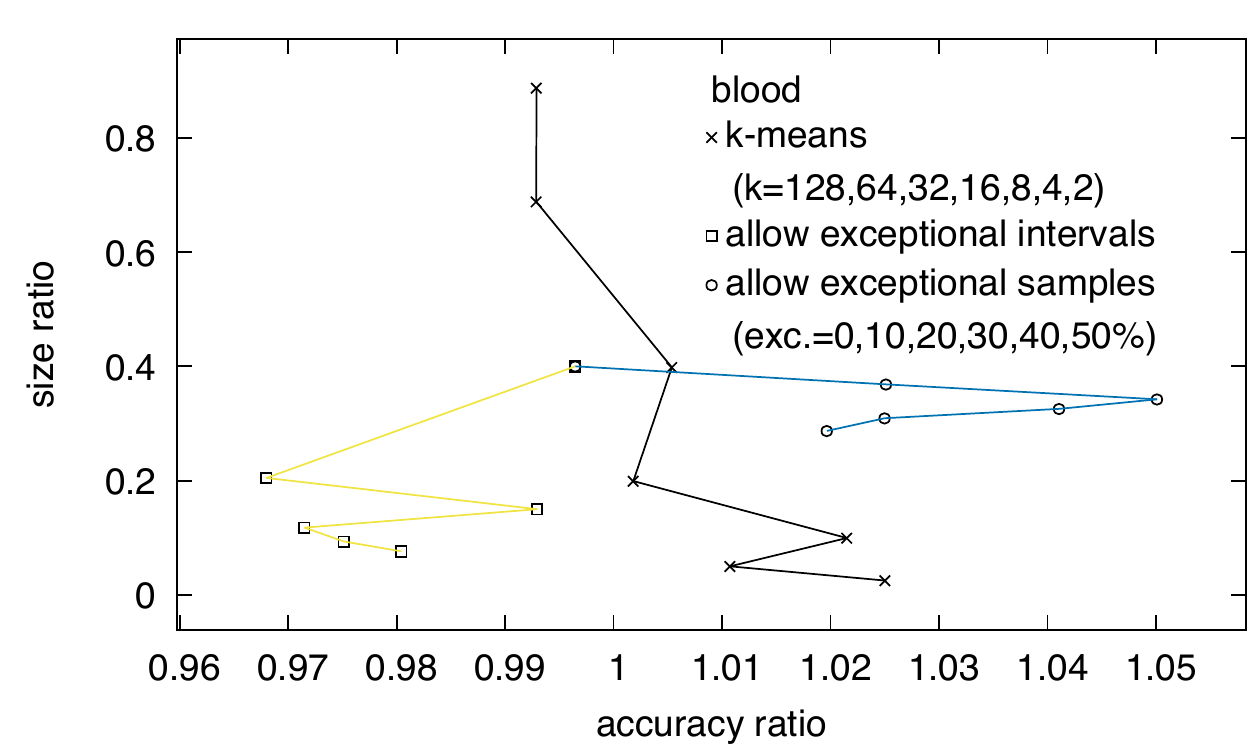}\includegraphics[width=0.33\linewidth]{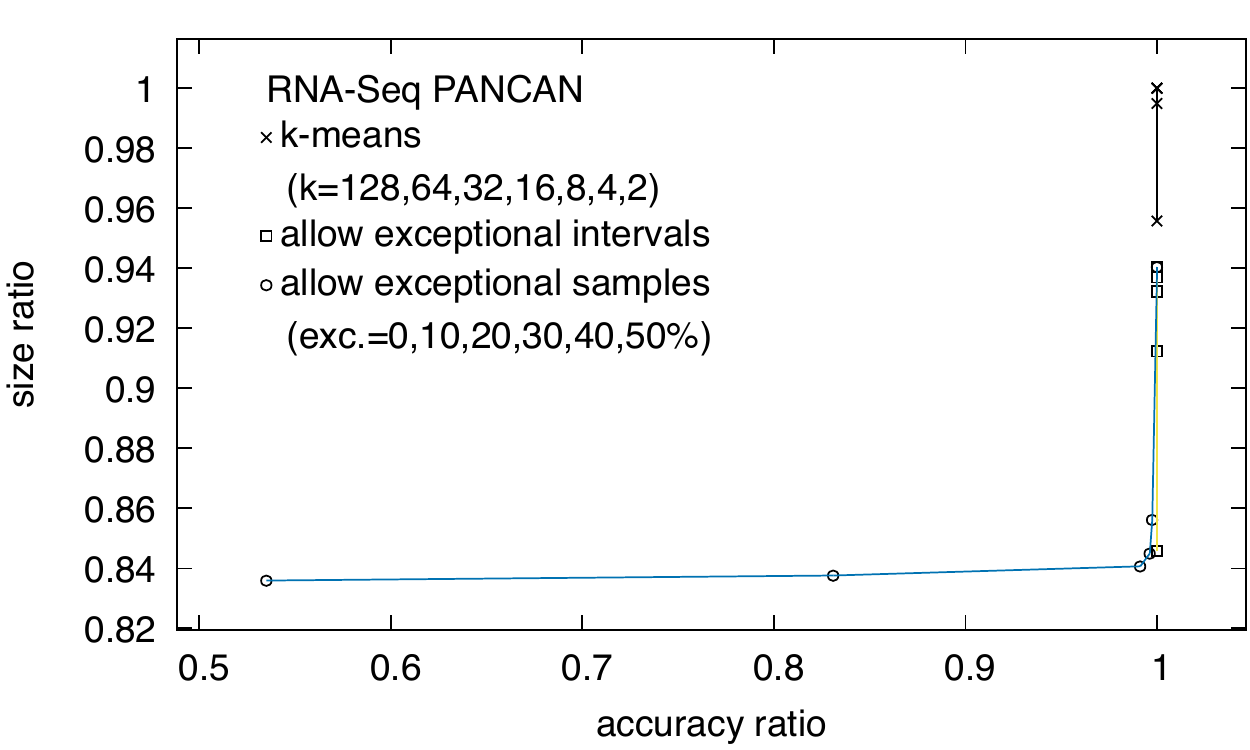}\includegraphics[width=0.33\linewidth]{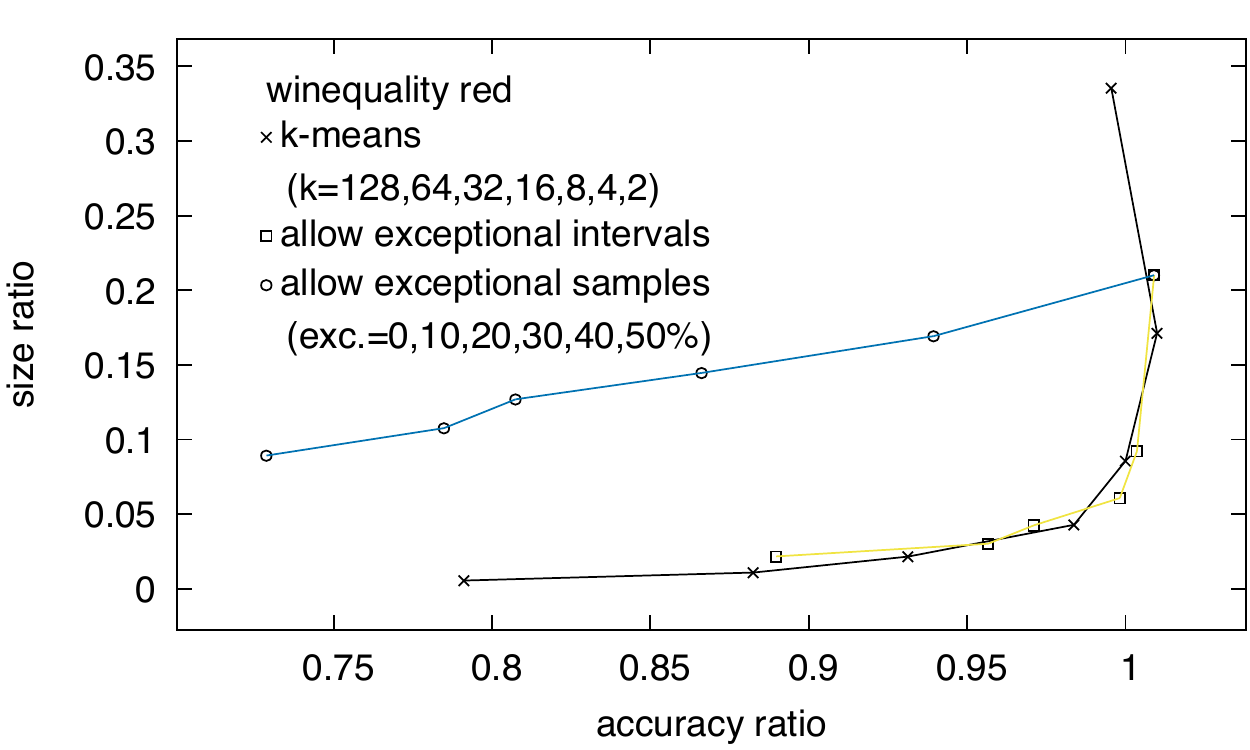}
  \includegraphics[width=0.33\linewidth]{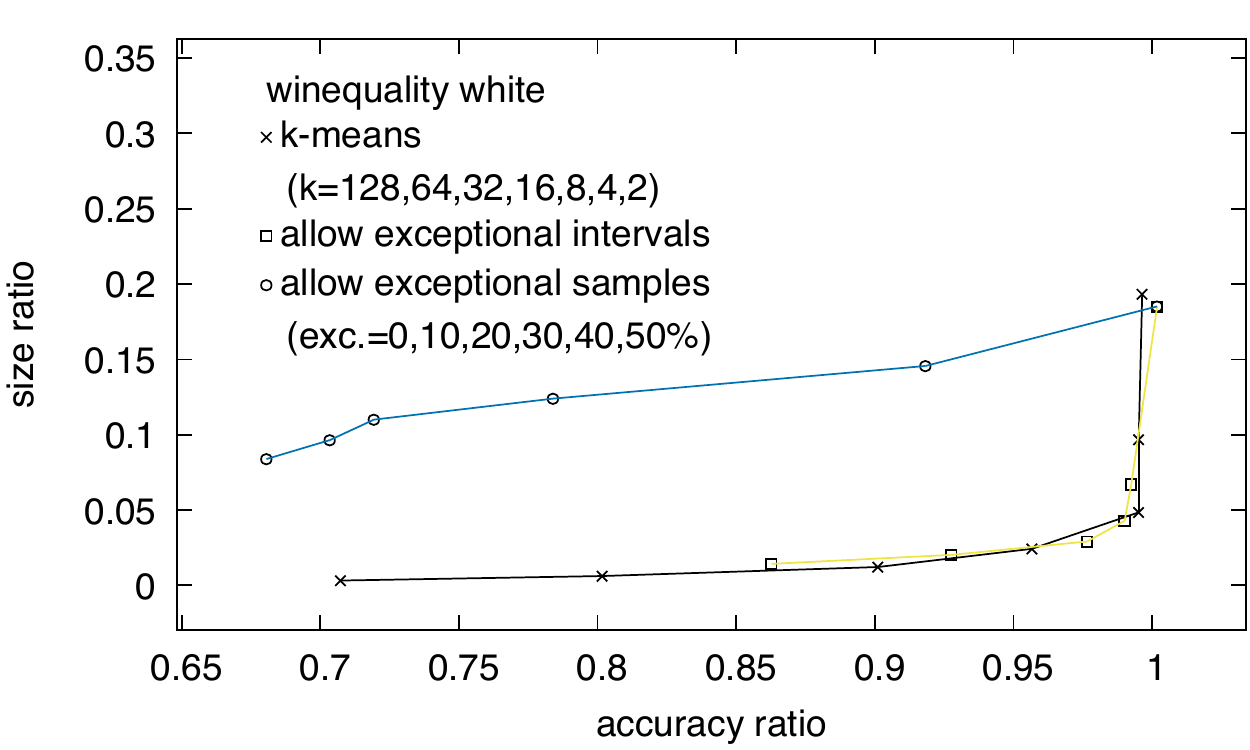}\includegraphics[width=0.33\linewidth]{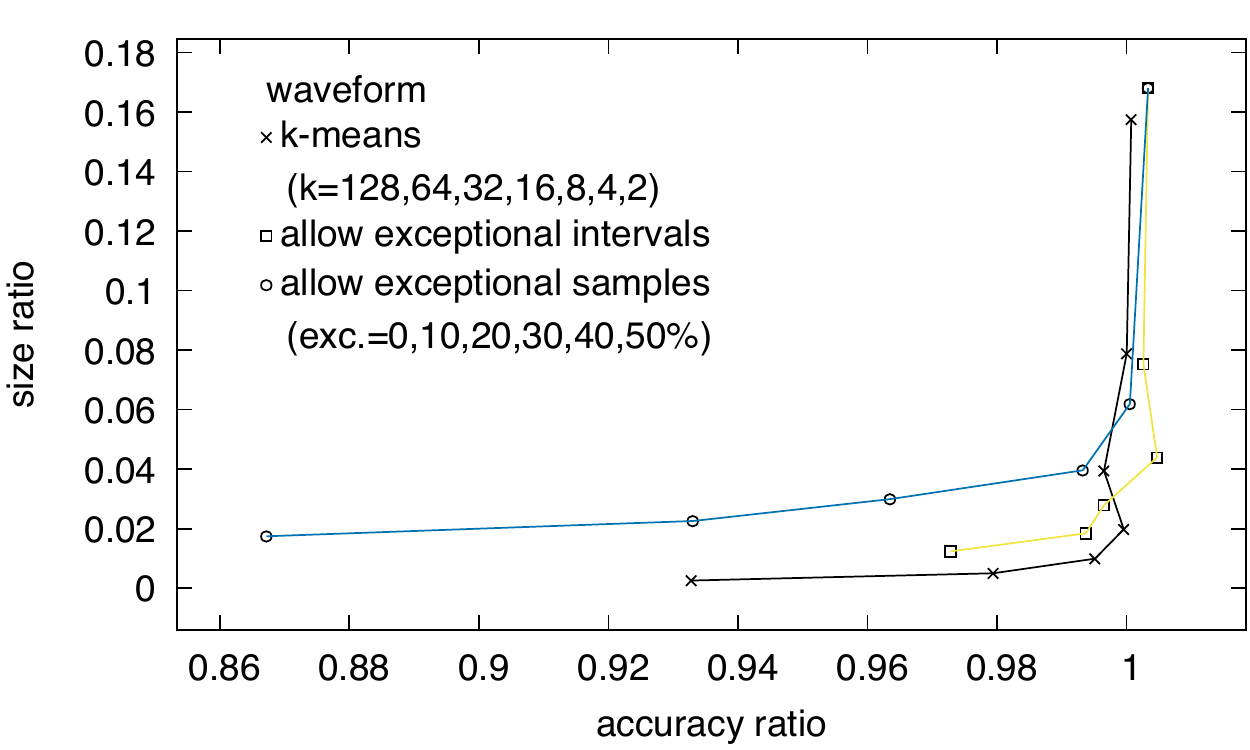}\includegraphics[width=0.33\linewidth]{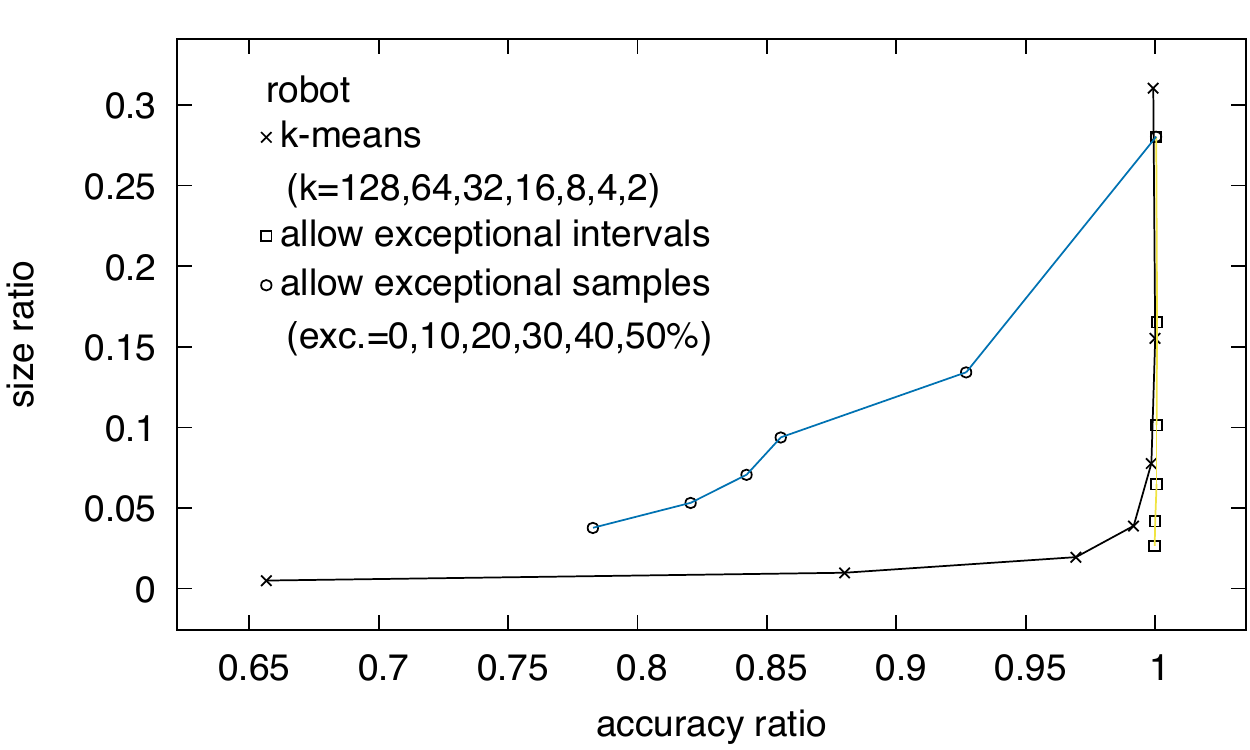}
  \includegraphics[width=0.33\linewidth]{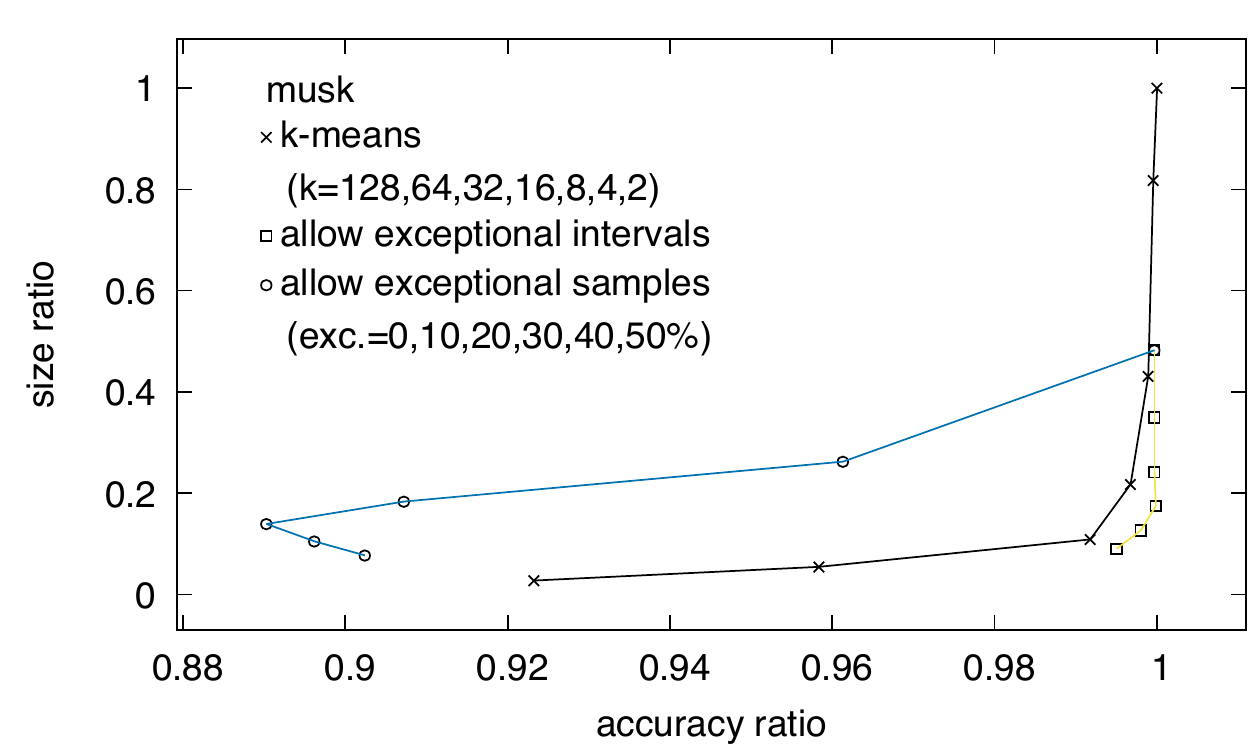}\includegraphics[width=0.33\linewidth]{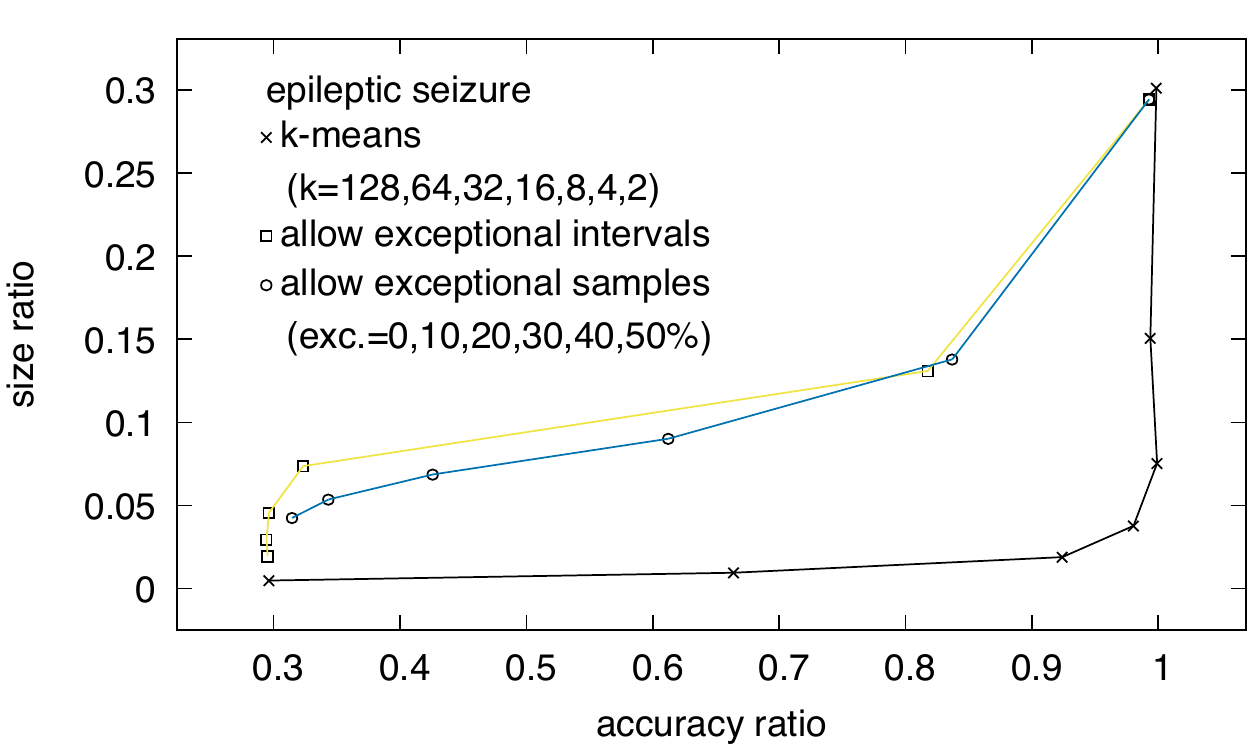}\includegraphics[width=0.33\linewidth]{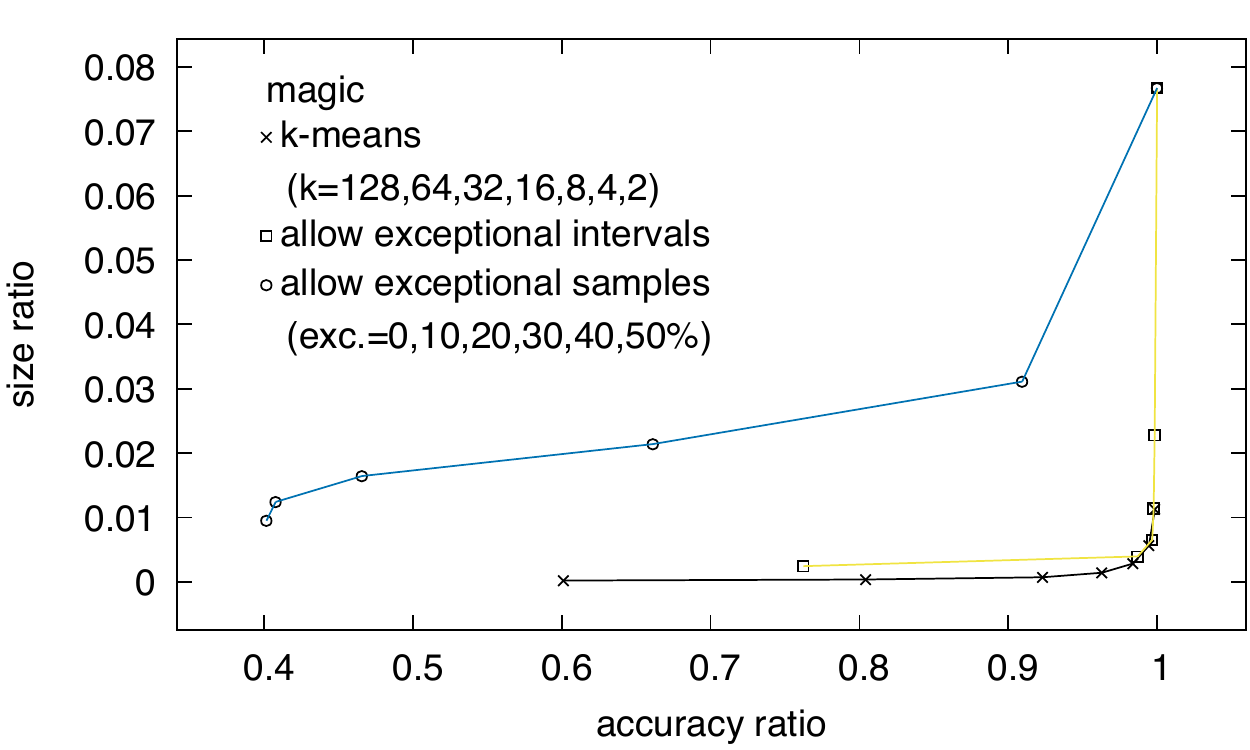}
  \includegraphics[width=0.33\linewidth]{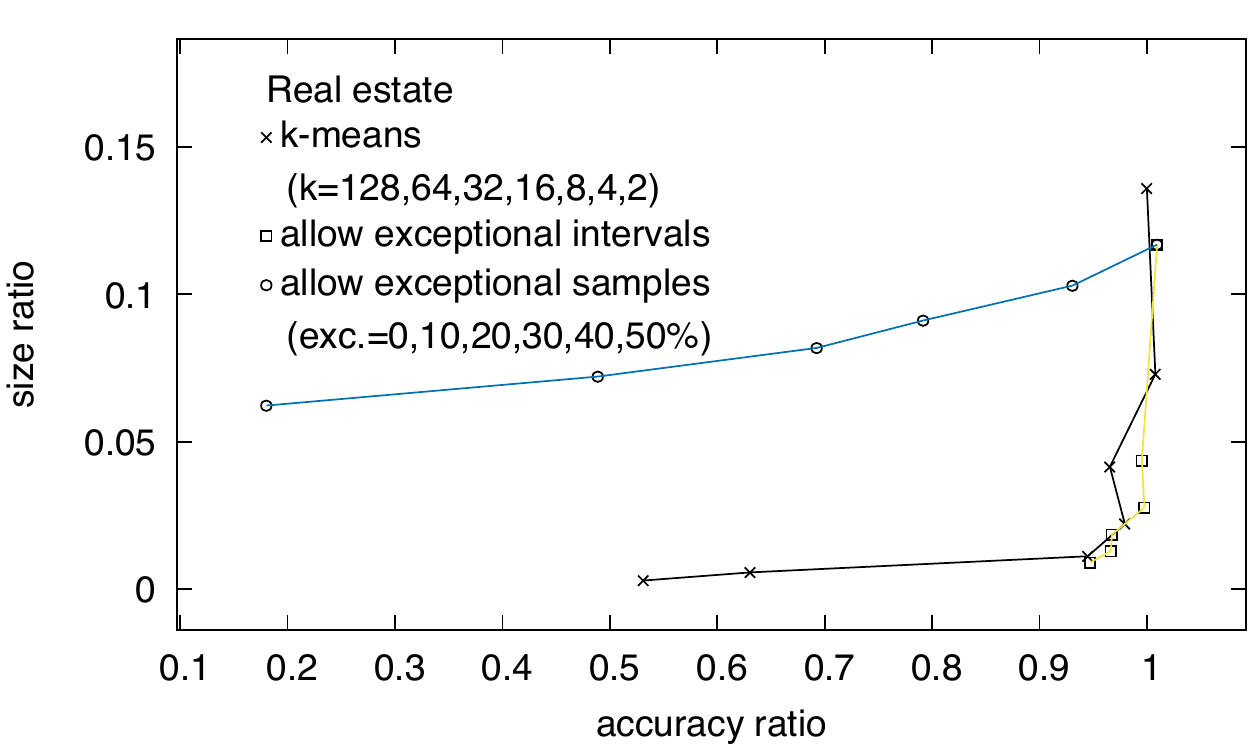}\includegraphics[width=0.33\linewidth]{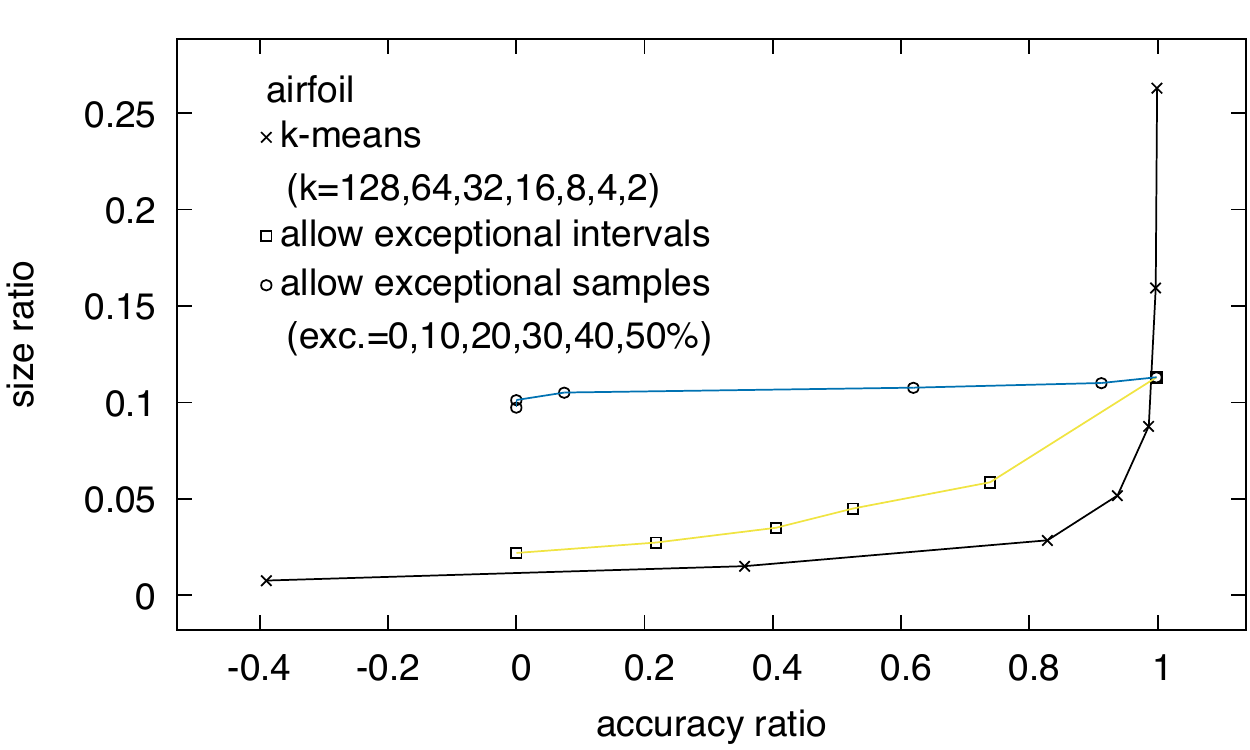}
  \caption{(accuracy ratio, size ratio)-relation line graphs for two extended Min\_DBN and the clustering-based method using k-means \citep{JSN2018}.}\label{fig:comp-kmeans}
\end{figure}

\begin{figure}[tb]
  \includegraphics[width=0.33\linewidth]{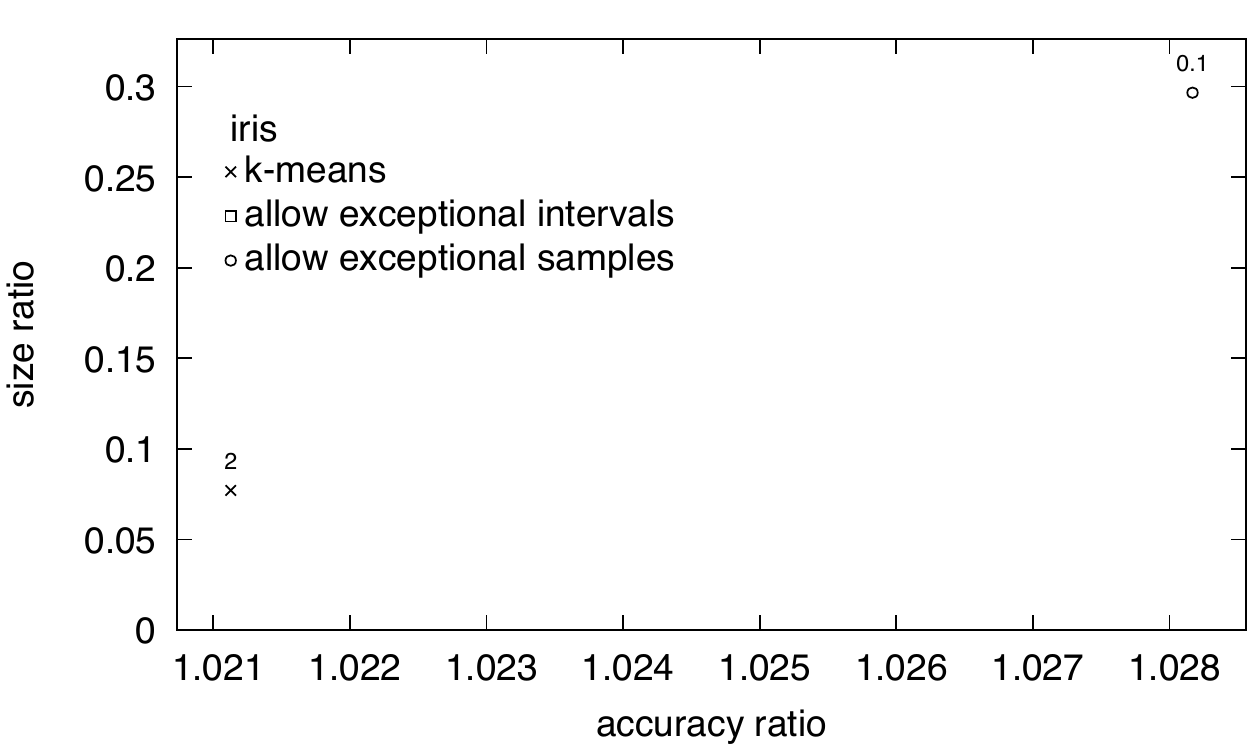}\includegraphics[width=0.33\linewidth]{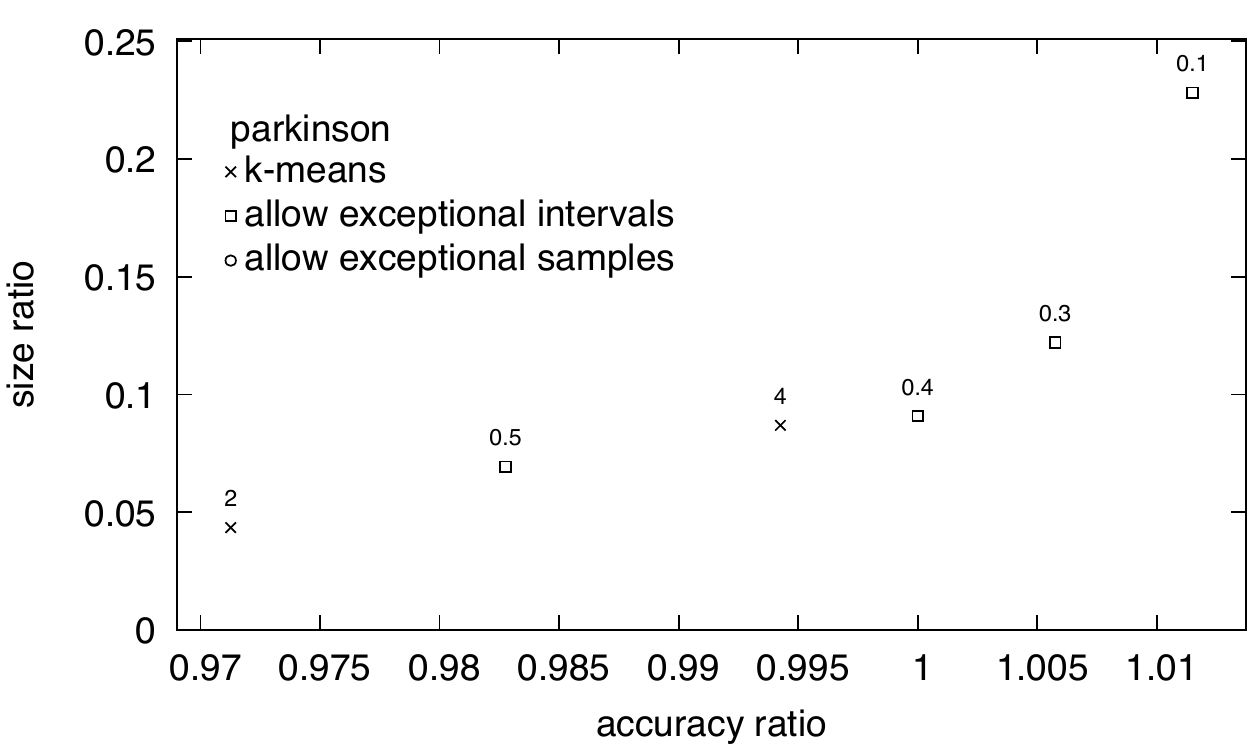}\includegraphics[width=0.33\linewidth]{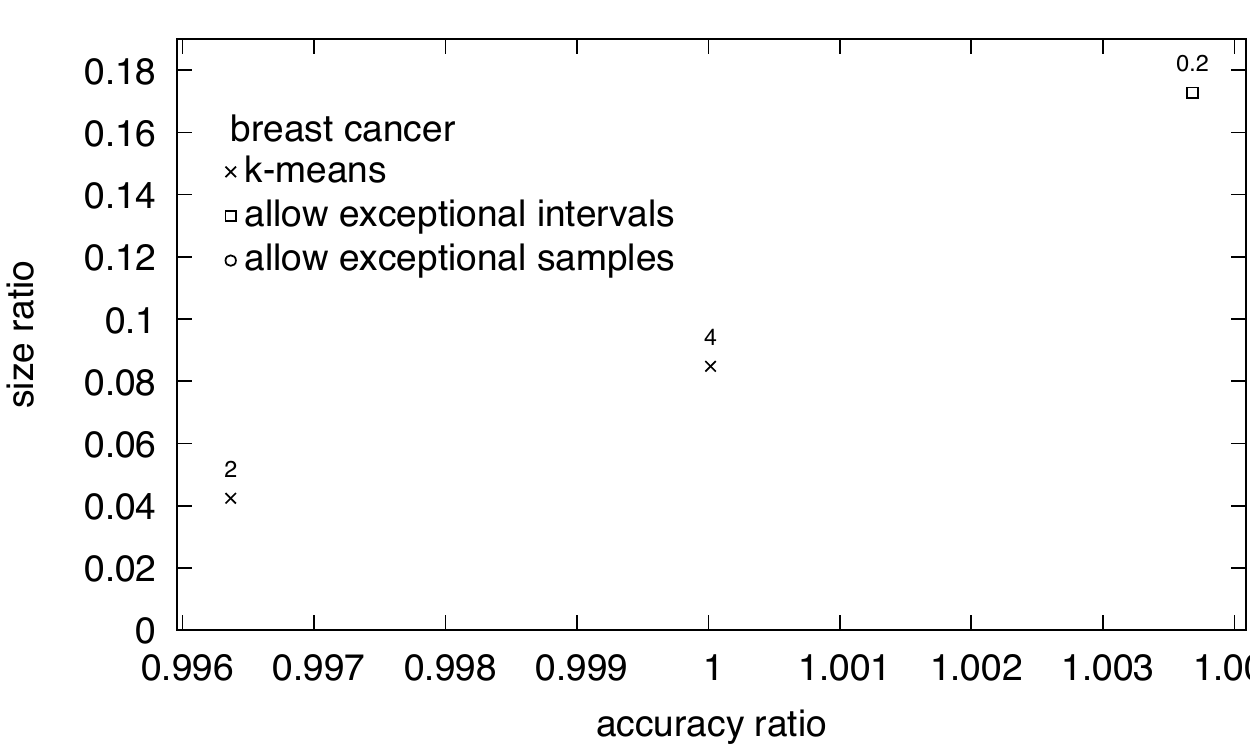}
  \includegraphics[width=0.33\linewidth]{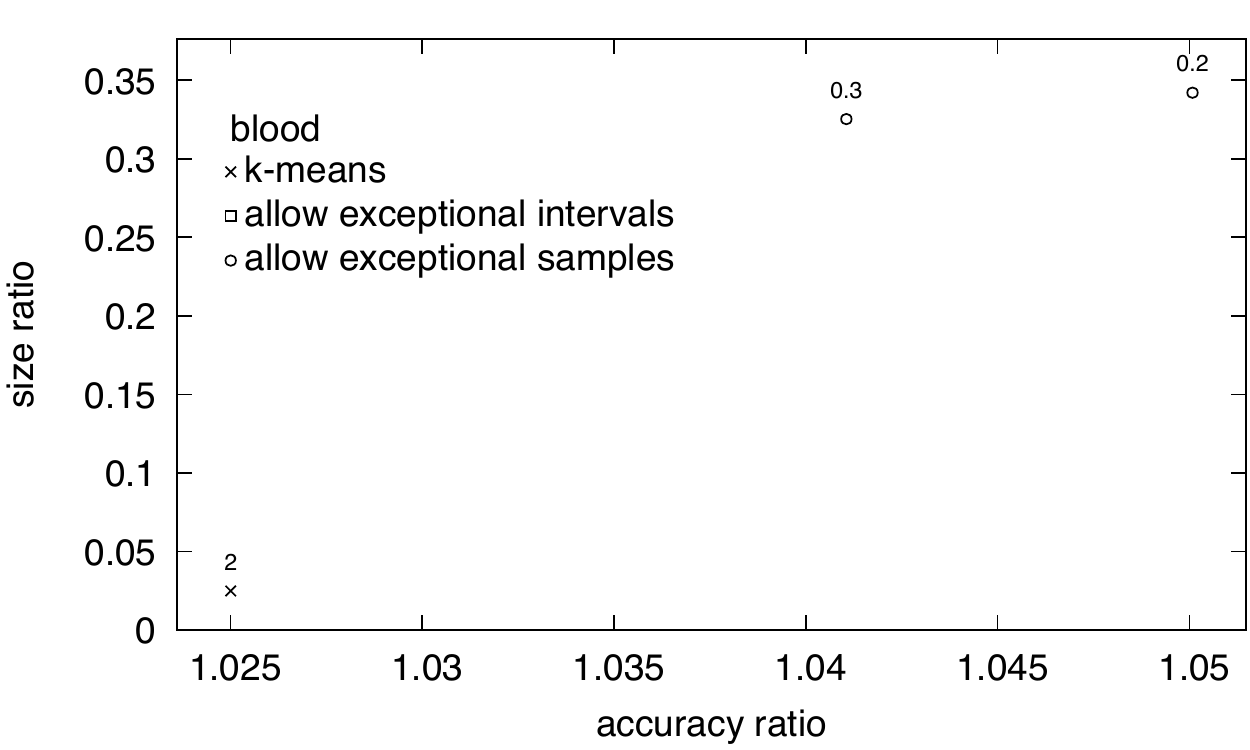}\includegraphics[width=0.33\linewidth]{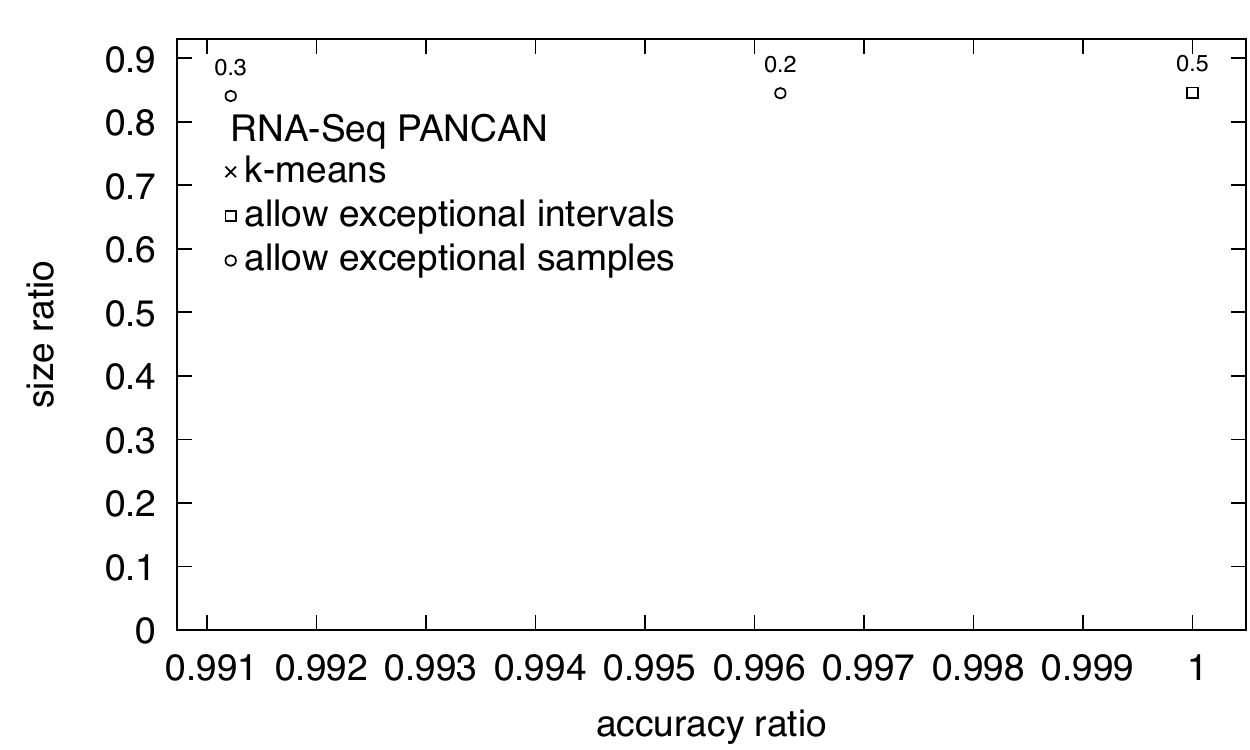}\includegraphics[width=0.33\linewidth]{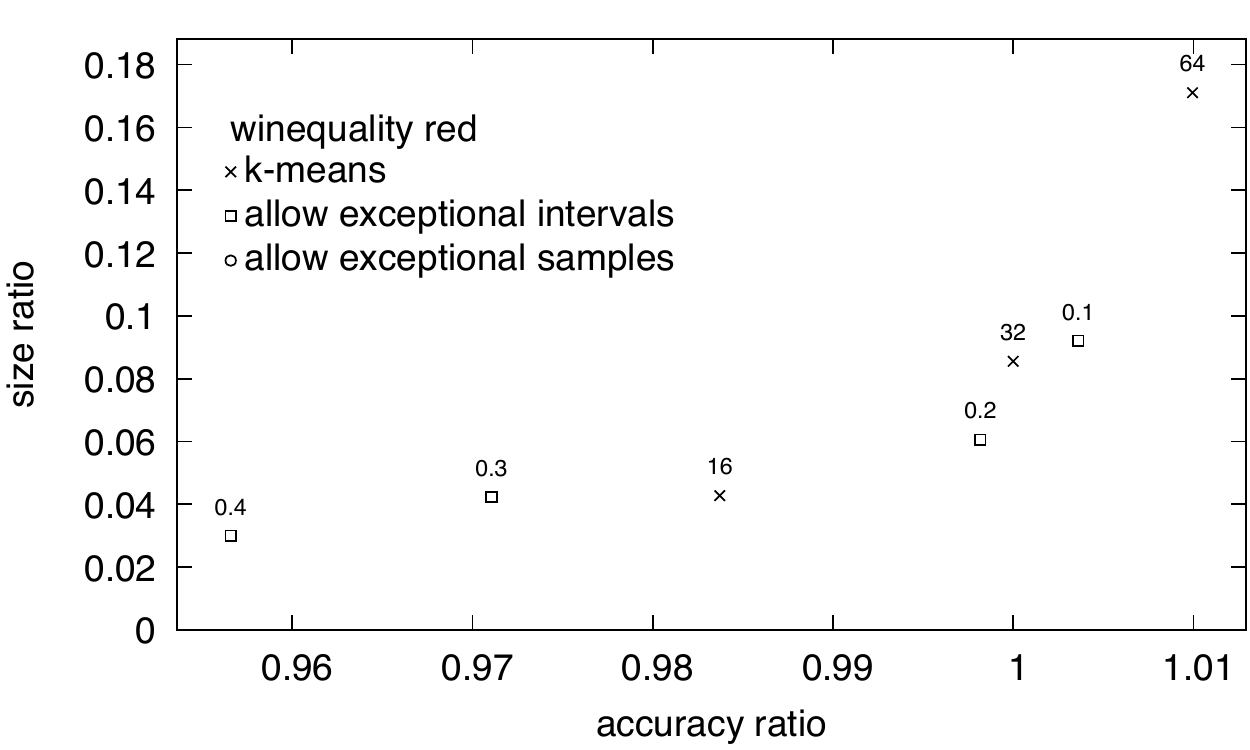}
  \includegraphics[width=0.33\linewidth]{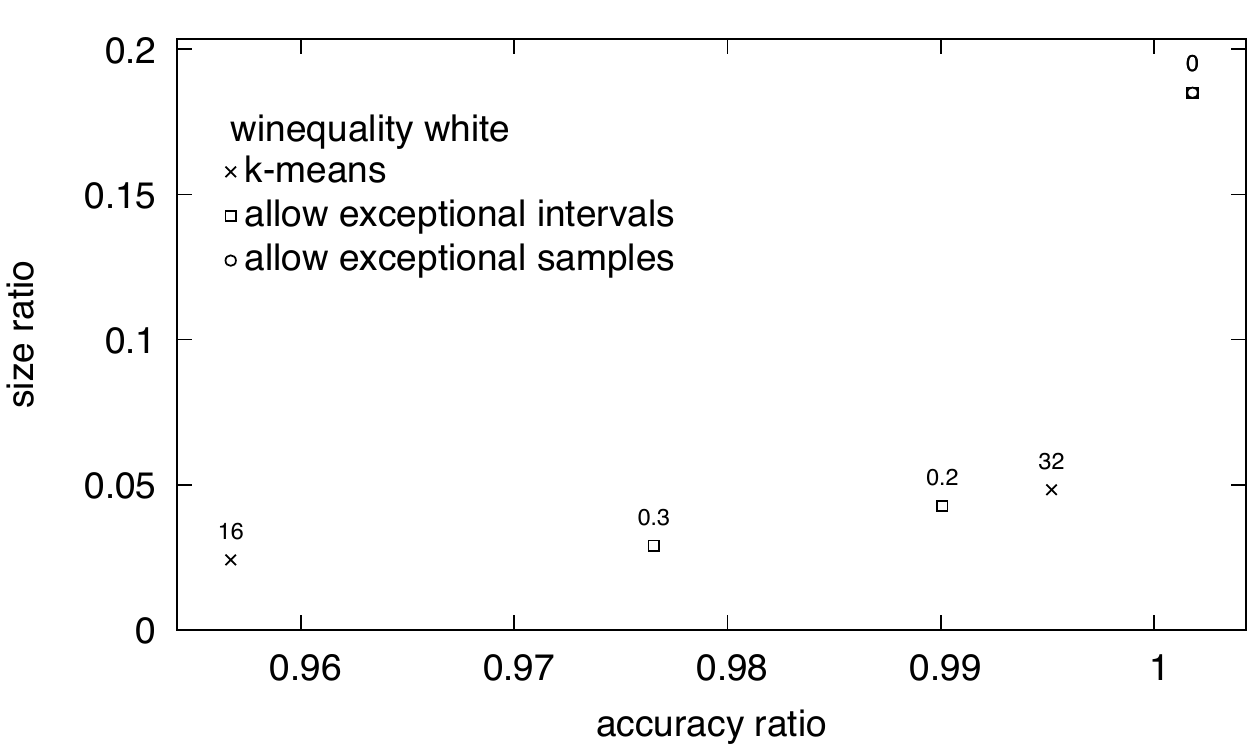}\includegraphics[width=0.33\linewidth]{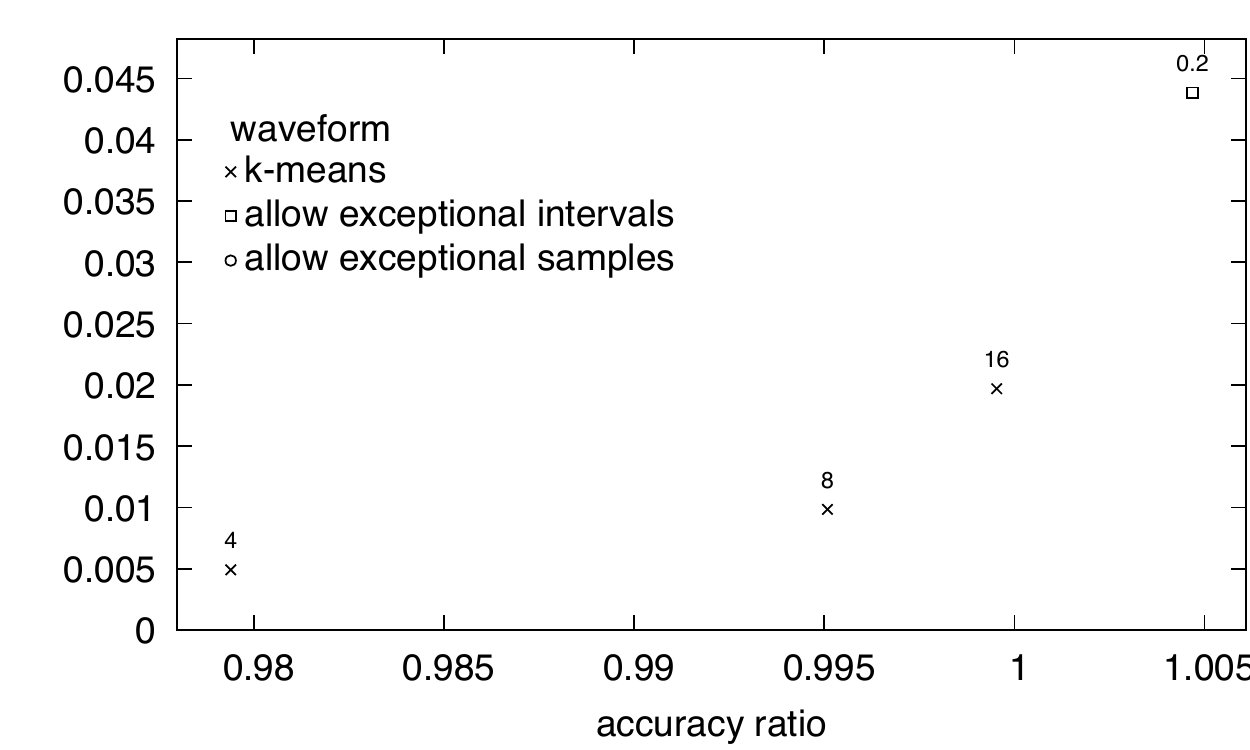}\includegraphics[width=0.33\linewidth]{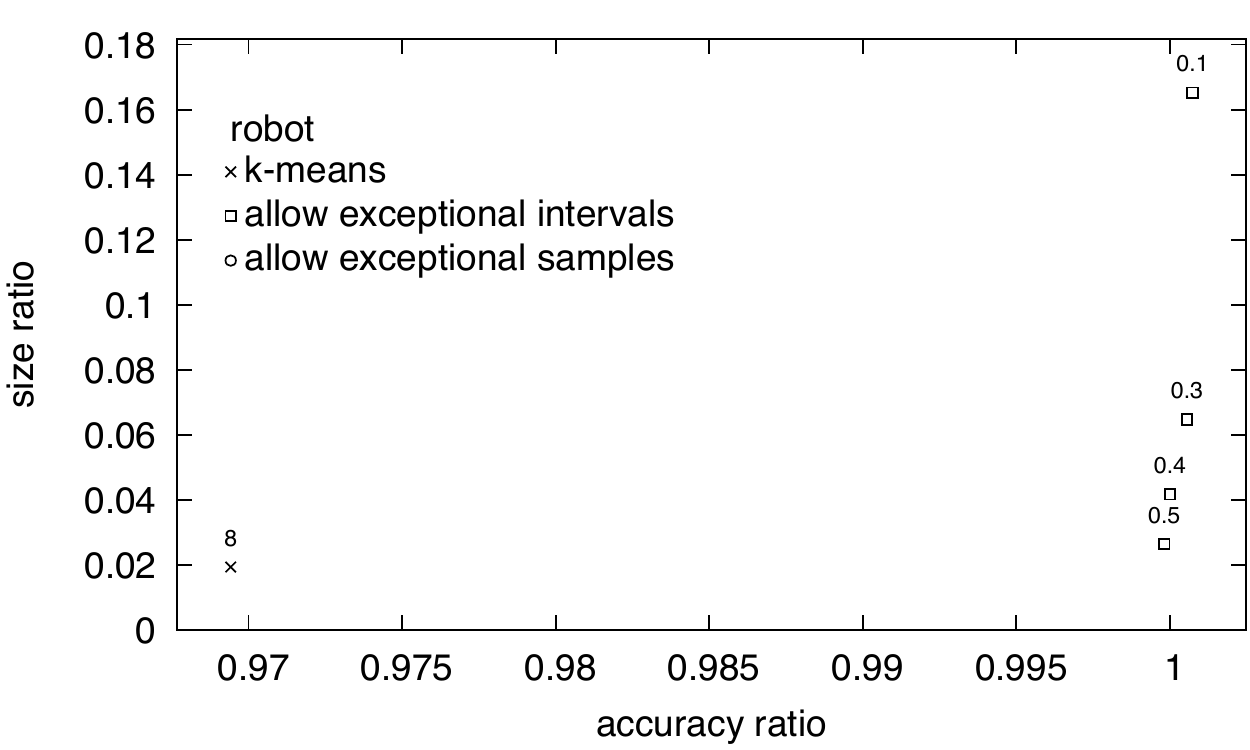}
  \includegraphics[width=0.33\linewidth]{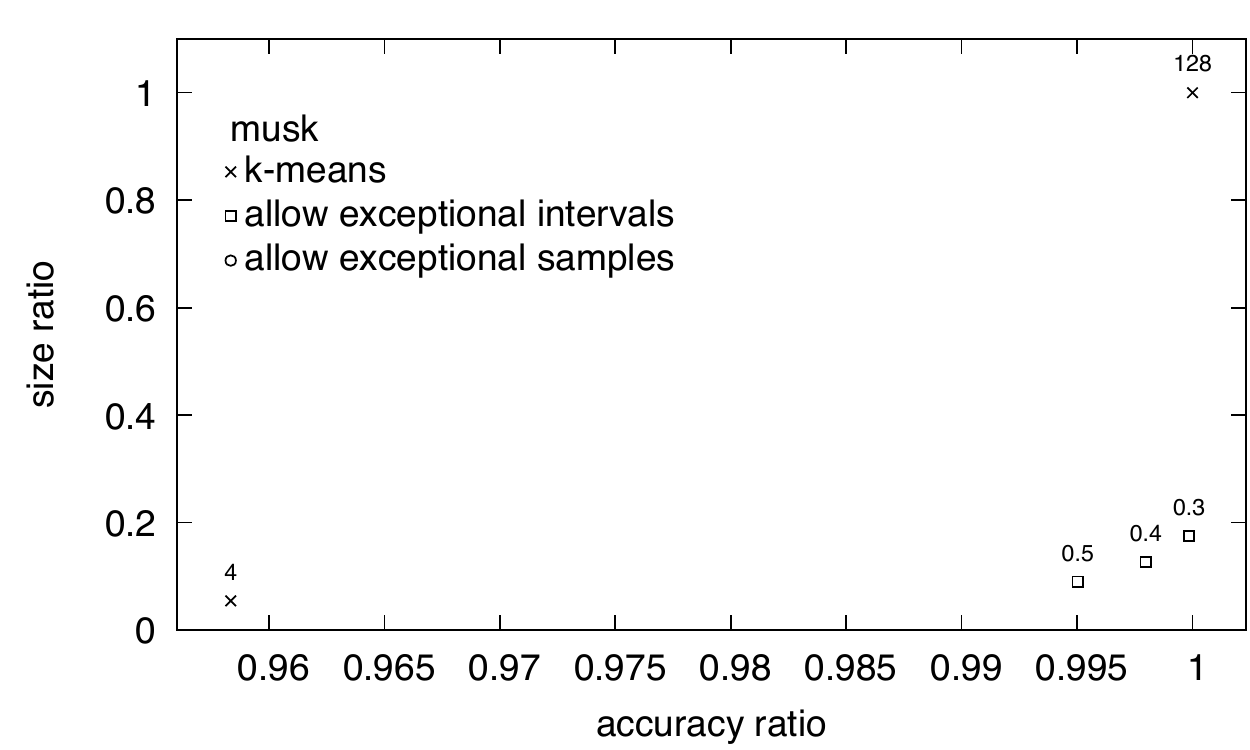}\includegraphics[width=0.33\linewidth]{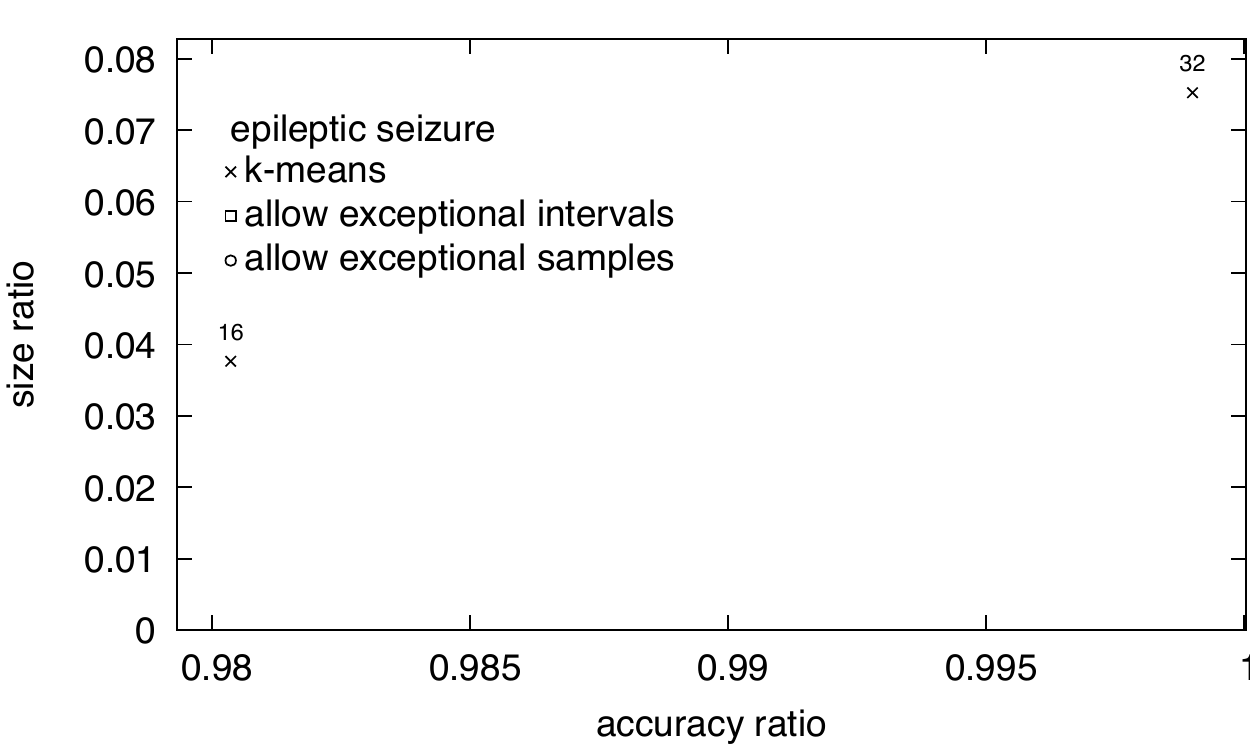}\includegraphics[width=0.33\linewidth]{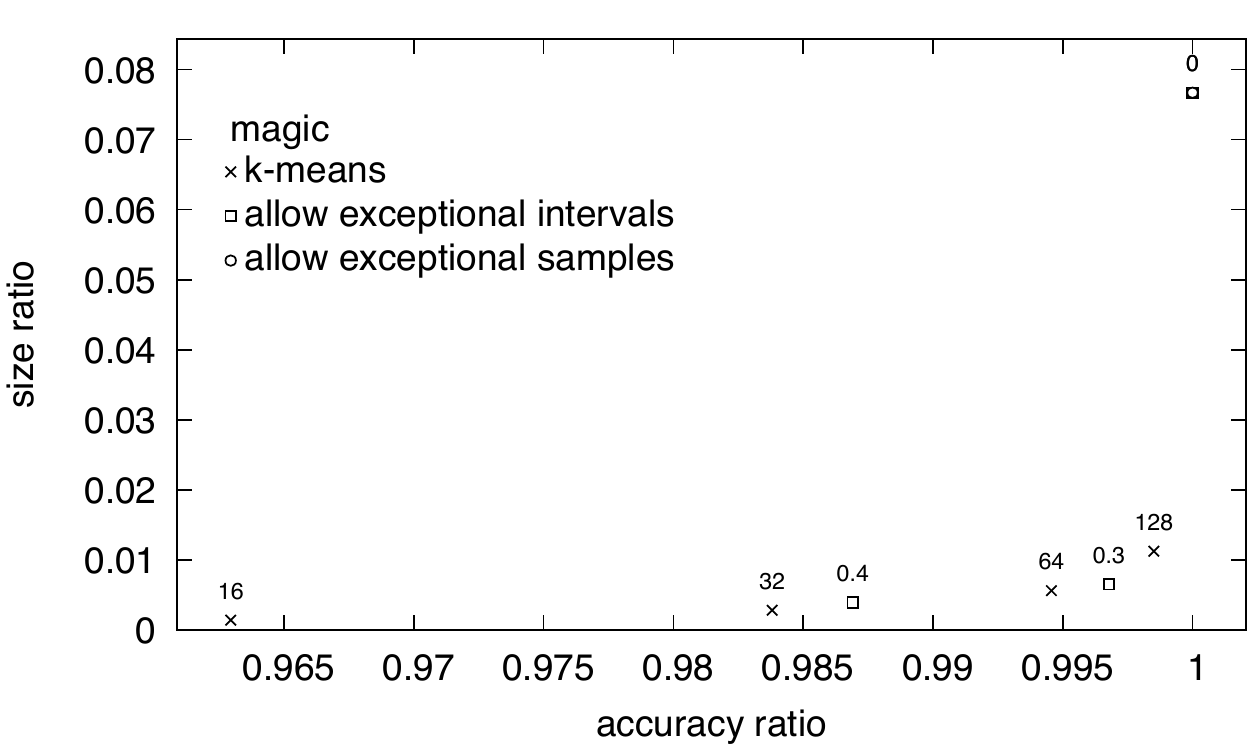}
  \includegraphics[width=0.33\linewidth]{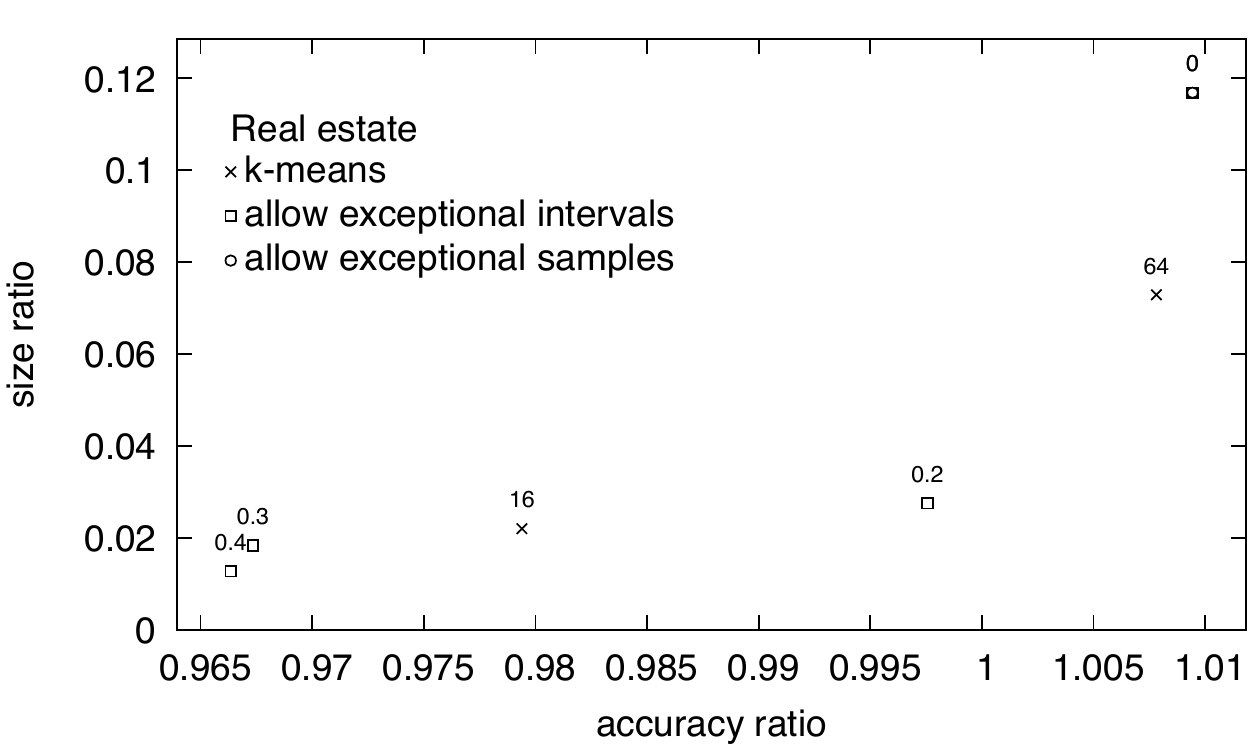}\includegraphics[width=0.33\linewidth]{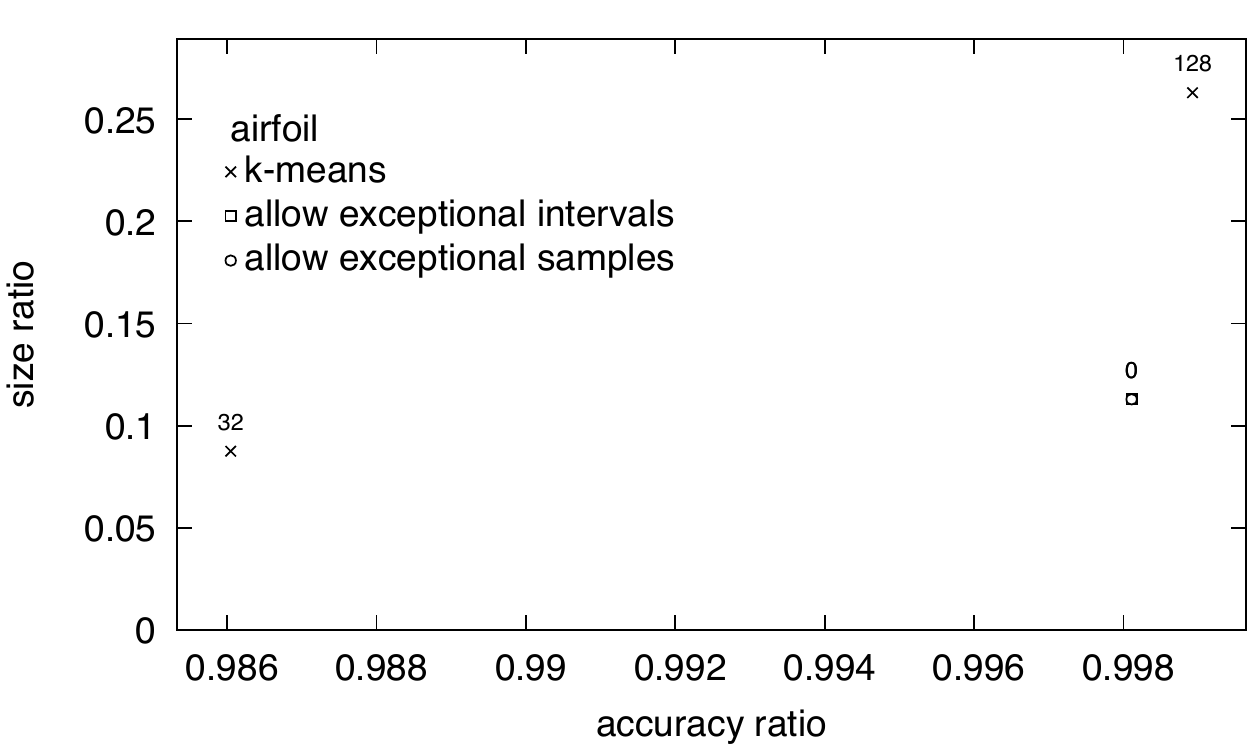}
  \caption{Pareto optimal predictors with accuracy ratio more than $0.95$ among the 18 predictors (the clustering-based method using k-means \citep{JSN2018} with $k=2,4,8,16,32,64,128$, modified Min\_DBN with path-changeable-rate $\sigma=0.1,0.2,0.3,0.4,0.5$, another modified Min\_DBN using Min\_IntSet\_wExc that allows $c=0.1,0.2,0.3,0.4,0.5\times$\#(intervals for each feature), and original Min\_DBN($\sigma=0$ or $c=0$)).}\label{fig:po-comp-kmeans}
\end{figure}

We conduct experiments to check the effectiveness of two extended versions of our branching condition sharing method.
One is an extension in which the decision paths of given feature vectors are allowed to change within a given path-changeable rate $\sigma$ at each branching node, that is, an extension \textbf{allowing exceptional samples}.
The other is an extension in which $c$ branching condition thresholds are allowed to be out of the path-invariant range
for samples passing through the node, that is, an extension \textbf{allowing exceptional intervals}.
For the 14 datasets that contain at most 20,000 instances,
we run two extended versions of Min\_DBN with parameters $100\sigma, 100c/p=0,10,20,30,40,50$,
where $p$ is the number of intervals for each feature, which depends on a dataset and a feature.
We also execute the existing clustering-based method \citep{JSN2018} in which the centers of $k$-means clustering
over the thresholds of branching conditions using the same feature are selected as the thresholds of shared branching conditions. In the experiment, parameter $k$ of the $k$-means clustering-based method is set to $2,4,8,16,32,64$, and $128$.

Figure~\ref{fig:comp-kmeans} shows the result of the experiments.
Note that the size reduces as exceptions increase for our two methods, and as $k$ decreases for $k$-means clustering-based method.
Comparing our two methods, the method allowing exceptional intervals performs better than the method allowing exceptional samples excluding two datasets (blood and epileptic seizure). In most datasets, accuracy degradation by allowing more sample exceptions is larger than that by allowing more interval exceptions though the latter version of Min\_DBN needs more memory and computational time than the former version.
As for the comparison to $k$-means clustering-based method, the method allowing exceptional interval method performs better for some datasets (parkinson, RNA-Seq PANCAN, robot, musk) but worse for some other datasets (blood, epileptic seizure, airfoil), and similar for the rest 7 datasets. Figure~\ref{fig:po-comp-kmeans} shows pareto optimal predictors among all the predictors with accuracy ratio more than 0.95 for each dataset, that is, predictors that have no predictors better than them in both accuracy and size ratio.
The method allowing interval exceptions finds more high accuracy-ratio pareto optimal predictors than the $k$-means clustering-based method.

\section{Conclusions}

We formalized novel simplification problems of a tree ensemble classifier or regressor, proposed algorithms for the problems,
and demonstrated their effectiveness for four major tree ensembles using 21 datasets in the UCI machine learning repository.
Our algorithm Min\_DBN can make a given tree ensemble share its branching conditions optimally under the constraint that
the decision path for any given feature vector does not change, which ensures accuracy preservation.
Simple clustering of branching condition thresholds is not easy to find the minimum number of cluster centers that surely keeps accuracy, though it more successfully shares branching conditions with high accuracy for some datasets. 
Loosening the constraint promotes our method more sharing of branching conditions and achieves a smaller size ratio of the number of distinct branching conditions without accuracy degradation for some datasets.

\section*{Acknowledgments}

We would like to thank Assoc. Prof. Ichigaku Takigawa, Assoc. Prof. Shinya Takamaeda-Yamazaki, Prof. Hiroki Arimura, and Prof. Masato Motomura for helpful comments to improve this research. 
This work was supported by JST CREST Grant Number JPMJCR18K3, Japan.

\vskip 0.2in
\bibliography{arXiv2022}

\end{document}